\documentclass[twoside]{article}

\usepackage[accepted]{aistats2020}
\usepackage[round]{natbib}

\usepackage{amsmath,amssymb,array,graphicx,verbatim,amsthm,wrapfig}

\newtheorem{theorem}{Theorem}
\newtheorem{lemma}{Lemma}

\newtheorem{corollary}{Corollary}

\newtheorem{assumption}{Assumption}
\newtheorem{remark}{Remark}
\usepackage{enumerate}
\usepackage{epstopdf}
\usepackage{dblfloatfix}
\usepackage{dsfont}
\usepackage{breqn}
\usepackage{varwidth}

\newcommand\norm[1]{\left\lVert#1\right\rVert}

\usepackage{url}

\usepackage{tikz}
\usetikzlibrary{arrows,shapes,calc,positioning}
\tikzstyle{block} = [draw,rectangle, rounded corners, minimum width=1cm, minimum height=0.8cm,text centered, line width=2pt ]
\tikzstyle{arrow} = [thick,->,>=stealth,line width=2pt]

\usepackage{circuitikz}
\usepackage{tikz-cd}
\usepackage{enumerate}
\tikzset{
  shift left/.style ={commutative diagrams/shift left={#1}},
  shift right/.style={commutative diagrams/shift right={#1}}
}
\usetikzlibrary{calc}

\usetikzlibrary{shapes}
\tikzstyle{block} = [draw,rectangle, rounded corners, minimum width=1cm, minimum height=0.8cm,text centered, line width=2pt ]
\tikzstyle{arrow} = [thick,->,>=stealth,line width=2pt]

\newcommand{\tp}{\intercal}		% transpose
\newcommand{\R}{\mathbb{R}}			% real numbers
\newcommand{\red}[1]{\textcolor{black}{#1}}

\DeclareMathOperator*{\argmin}{arg\,min\,}

\DeclareMathOperator{\ee}{\mathbb{E}}			% expected value
\DeclareMathOperator{\prob}{\mathbb{P}}			% probability
\DeclareMathOperator{\vecc}{\mathbf{vec}}		% probability
\DeclareMathOperator{\tr}{\mathbf{tr}}			% trace
\DeclareMathOperator{\cov}{\mathbf{cov}}		% covariance
\usepackage [english]{babel}
\usepackage [autostyle, english = american]{csquotes}
\MakeOuterQuote{"}

\usepackage{algorithm}
\usepackage{algpseudocode}

\usepackage[font=small,labelfont=bf,tableposition=top]{caption}
\usepackage[font=footnotesize]{subcaption}

\begin{document}



\twocolumn[

\aistatstitle{Regret Bounds for Decentralized Learning in Cooperative Multi-Agent Dynamical Systems}

%\aistatsauthor{ Author 1 \And Author 2 \And  Author 3 }
%
%\aistatsaddress{ Institution 1 \And  Institution 2 \And Institution 3 } 

\aistatsauthor{Seyed Mohammad Asghari \And Yi Ouyang \And  Ashutosh Nayyar}

\aistatsaddress{ University of Southern California \And  Preferred Networks America, Inc \And University of Southern California }
 ]

\begin{abstract}

Regret analysis is challenging in Multi-Agent Reinforcement Learning (MARL) primarily due to the dynamical environments and the decentralized information among agents. We attempt to solve this challenge in the context of decentralized learning in multi-agent linear-quadratic (LQ) dynamical systems.
We begin with a simple setup consisting of two agents and two dynamically decoupled stochastic linear systems, each system controlled by an agent. The systems are coupled through a quadratic cost function. When both systems' dynamics are unknown and there is no communication among the agents, we show that no learning policy can generate sub-linear in $T$ regret, where $T$ is the time horizon.   When only one system's dynamics are unknown and there is one-directional communication from the agent controlling the unknown system to the other agent, %system to the other one, 
we propose a MARL algorithm based on the construction of an auxiliary single-agent LQ problem.
The auxiliary single-agent problem in the proposed MARL algorithm serves as an implicit coordination mechanism among the two learning agents. This allows the agents to  achieve a regret within $O(\sqrt{T})$ of the regret of the auxiliary single-agent problem. 
Consequently, using existing results for single-agent LQ regret, our algorithm provides a $\tilde{O}(\sqrt{T})$ regret bound. (Here $\tilde{O}(\cdot)$ hides constants and logarithmic factors). Our numerical experiments indicate that this bound is matched in practice. From the two-agent problem, we extend our results to multi-agent LQ systems with certain communication patterns.

\end{abstract}

\section{Introduction}

Multi-agent systems arise in many different domains, including multi-player card games \citep{bard2019hanabi}, robot teams \citep{stone1998team}, vehicle formations \citep{fax2004information}, urban traffic control \citep{oliveira2010multi}, and power grid operations \citep{schneider1999distributed}.
A multi-agent system consists of multiple autonomous agents operating in a common environment.
Each agent gets observations from the environment (and possibly from some other agents) and, based on these observations, each agent chooses actions to collect rewards from the environment.
The agents' actions  may  influence the environment dynamics  and the reward of each agent. 
Multi-agent systems where the environment model is known to all agents have been considered under the frameworks of multi-agent planning \citep{oliehoek2016concise}, decentralized optimal control \citep{yuksel2013stochastic}, and non-cooperative game theory \citep{basar1999dynamic}.
In realistic situations, however, the environment model is usually only partially known or even totally unknown.
Multi-Agent Reinforcement Learning (MARL) aims to tackle the general situation of multi-agent sequential decision-making where the environment model is not completely known to the agents. In the absence of the environmental model, each agent needs to learn  the environment while interacting with it to collect rewards. 
%\red{Should this be here: See \citep{panait2005cooperative, bu2008comprehensive, nowe2012game, hernandez2017survey, hernandez2018multiagent} for surveys on MARL.}
In this work, we focus on decentralized learning in a cooperative multi-agent setting where all agents share the same reward (or cost) function.

A number of successful learning algorithms have been developed for Single-Agent Reinforcement Learning (SARL) in single-agent environment models such as finite Markov decision processes (MDPs) and linear quadratic (LQ) dynamical systems.
To extend SARL algorithms to cooperative MARL problems, one key challenge is the coordination among agents \citep{panait2005cooperative,hernandez2017survey}. In general, agents have access to different information and hence agents may have different views about the environment from their different learning processes. This difference in perspectives makes it difficult for agents to coordinate their actions for maximizing rewards. 
%Information available to multiple agents are generally decentralized.
%Different agents have access to different data and each agent can only learn from its own experience.
%So the agents may have different views about the environment from their different learning processes.

One popular method to resolve the coordination issue is to have a central entity collect information from all agents and determine the policies for each agent.
Several works generalize SARL methods to multi-agent settings with such an approach by either assuming the existence of a central controller or by training a centralized agent with information from all agents in the learning process, which is the idea of \textit{centralized training with decentralized execution} \citep{foerster2016learning, dibangoye2018learning, hernandez2018multiagent}.
With centralized information, the learning problem reduces to a single-agent problem which can be readily solved by SARL algorithms.
In many real-world scenarios, however, there does not exist a central controller or a centralized agent receiving all the information.
Agents have to learn in a decentralized manner based on the observations they get from the environment and possibly from some other agents. In the absence of a centralized entity, an efficient MARL algorithm should guide each agent to learn the environment while maintaining certain level of coordination among agents. 
%their   interactions with the environment and possible communication from other agents.

Moreover, in online learning scenarios, the trade-off between exploration and exploitation is critical for the performance of a MARL algorithm during learning \citep{hernandez2017survey}.
Most existing SARL algorithms balance the exploration-exploitation trade off by controlling the posterior estimates/beliefs of the agent.
Since multiple agents have decentralized information in MARL, it is not possible to directly extend SARL methods given the agents' distinct posterior estimates/beliefs.
Furthermore, the fact that each agent's estimates/beliefs may be private to itself prevents  any direct imitation of SARL. These issues make it extremely challenging to design coordinated policies for multiple agents to learn the environment and maintain good performance  during learning.
In this work, we attempt to solve this challenge in online decentralized MARL  in the context of multi-agent learning in linear-quadratic (LQ) dynamical systems.
Learning in LQ systems is an ideal benchmark for studying MARL due to a combination of its theoretical tractability and its practical application in various engineering domains \citep{astrom2010feedback, abbeel2007application, levine2016end, abeille_lqg_portfolio, lazic2018data}. 

We begin with a simple setup consisting of two agents and two stochastic linear systems as shown in Figure \ref{fig:SystemModel}. The systems are dynamically decoupled but coupled through a quadratic cost function. In spite of its simplicity, this setting  illustrates some of the inherent challenges and potential results in MARL. When the parameters of both systems 1 and 2 are known to both agents, the optimal solution to this multi-agent control problem can be computed in closed form \citep{ouyang2018optimal}. 
We consider the settings where the system parameters are completely or partially unknown and formulate an online MARL problem  to minimize the agents' regret during learning.
The regret is defined to be the difference between the  cost incurred by the learning agents and the steady-state cost of the optimal policy computed using  complete knowledge of the system parameters. 

We provide a finite-time regret analysis for a decentralized MARL problem with  controlled dynamical systems. 
In particular, we show that
\begin{enumerate}
\item First, if all parameters of a system are unknown, then both agents should receive information about the state of this system; otherwise, there is \textbf{no} learning policy that can guarantee sub-linear regret for all instances of the decentralized MARL problem (Theorem 1 and Lemma 2).
\item Further, when only one system's dynamics are unknown and there is one-directional communication from the agent controlling the unknown system to the other agent, we propose a MARL algorithm with regret bounded by $\tilde{O}(\sqrt{T})$ (Theorem 2 and Corollary 1).
\end{enumerate}

The proposed MARL algorithm builds on an auxiliary SARL problem constructed  from the MARL problem.
Each agent constructs the auxiliary SARL problem by itself and applies a SARL algorithm $\mathcal{A}$ to it. Each agent chooses its action  by modifying the output of the SARL algorithm $\mathcal{A}$ based on its information at each time.
\red{
In our proposed algorithm, the auxiliary SARL problem serves as the critical coordination tool for the two agents to learn individually while jointly maintaining an exploration-exploitation balance.
In fact, we will later show that the SARL dynamics can be seen as the filtering equation for the common state estimate of the agents.
}

We show that the regret achieved by our MARL algorithm is upper bounded by the regret of the SARL algorithm $\mathcal{A}$ in {the auxiliary SARL problem} plus an overhead bounded by $O(\sqrt{T})$. This implies that the MARL regret can be bounded by $\tilde{O}(\sqrt{T})$ by letting $\mathcal{A}$ be one of   the state-of-the-art SARL algorithms for LQ systems which achieve $\tilde{O}(\sqrt{T})$ regret \citep{abbasi2011regret,ibrahimi2012efficient, faradonbeh2017regret, faradonbeh2019applications}.
Our numerical experiments indicate that this bound is matched in simulations.
From the two-agent problem, we  extend our results to multi-agent LQ systems with certain communication patterns. 
%%%%%%%%%
%%%%%%%%%

\textbf{Related work.}
There exists a rich and expanding body of work in the field of MARL \citep{littman1994markov,bu2008comprehensive, nowe2012game}.
Despite recent successes in empirical works including the adaptation of deep learning \citep{hernandez2018multiagent}, many theoretical aspects of MARL are still under-explored.
As multiple agents learn and adapt their policies, the environment is non-stationary from a single agent's perspective \citep{hernandez2017survey}.
Therefore, convergence guarantees of SARL algorithms are mostly invalid for MARL problems.
Several works have extended SARL algorithms to independent or cooperative agents
and analyzed their convergence properties \citep{tan1993multi, greenwald2003correlated, 
matignon2012independent, kar2013cal, amato2015scalable, zhang2018fully, gagrani2018thompson, wai2018multi}.
However, most of these works do not take into account the performance during learning except \citet{bowling2005convergence}.
%The concept of no-regret learning was introduced in \citep{bowling2005convergence} where the goal is to achieve zero average regret in the limit.
The algorithm of \citet{bowling2005convergence} has a regret bound of $O(\sqrt{T})$, but the analysis is limited to repeated games.
In contrast, we are interested in MARL in dynamical systems.

Regret analysis in online learning has been mostly focusing on multi-armed bandit (MAB) problems \citep{lai1985asymptotically}.
Upper-Confidence-Bound (UCB) \citep{auer2002finite,bubeck2012regret,dani2008stochastic} and Thompson Sampling \citep{thompson1933likelihood,kaufmann2012thompson,agrawal2013thompson, russo2014learning} are the two popular classes of algorithms that provide near-optimal regret guarantees in single-agent MAB.
These ideas have been extended to certain multi-agent MAB settings \citep{liu2010distributed,korda2016distributed,nayyar2016regret}. Multi-agent MAB can be viewed as a special class of MARL problems, but the lack of dynamics in MAB environments makes a drastic difference from the dynamical setting in this paper. 

In the learning of dynamical systems, recent works have adopted concepts from MAB to analyze the regret of SARL algorithms in MDP \citep{jaksch2010near, osband2013more,gopalan2015thompson,ouyang2017learning_nips} and LQ systems \citep{
abbasi2011regret,ibrahimi2012efficient, faradonbeh2017regret,faradonbeh2019applications, ouyang2017learning, 
faradonbeh2018optimality, abbasi2015bayesian, abeille2018improved}.
Our MARL algorithm builds on these SARL algorithms by using  the novel idea of constructing an auxiliary SARL problem for multi-agent coordination.

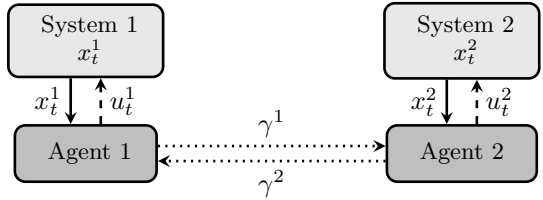
\begin{figure}
\begin{center}
\begin{tikzpicture}
\node [rectangle,draw,minimum width=1.2cm,minimum height=0.75cm,line width=1pt,rounded corners, fill=black!25]at (-1.0,0) (1) {
\begin{small}
\begin{tabular}{c}
%Remote \\ 
Agent 1
\end{tabular}
\end{small}}; 
\node [rectangle,draw,minimum width=1.2cm,minimum height=0.75cm,line width=1pt,rounded corners,fill=black!25]at (4.0,0) (2) {
\begin{small}
\begin{tabular}{c}
%Local \\ 
Agent 2
\end{tabular}
\end{small}
}; 
\node [rectangle,draw,minimum width=1.5cm,minimum height=0.75cm,line width=1pt,rounded corners,fill=black!9]at (-1.0,1.5) (3){\begin{small}
\begin{tabular}{c}
System 1 \\ $x_t^1$
\end{tabular}
\end{small}};
\node [rectangle,draw,minimum width=1.5cm,minimum height=1cm,line width=1pt,rounded corners,fill=black!9]at (4.0,1.5) (4) {\begin{small}
\begin{tabular}{c}
System 2 \\ $x_t^2$
\end{tabular}
\end{small}}; 

\path[thick,->,>=stealth, dashed, line width=1pt]
	($(1.north) + (0.2,0)$) edge node{} ($(3.south) + (0.2,0)$)
	($(2.north) + (0.2,0)$) edge node{} ($(4.south) + (0.2,0)$)	
	;
\path[thick,->,>=stealth,line width=1pt]
	($(3.south) + (-0.2,0)$) edge node{} ($(1.north) + (-0.2,0)$)
	($(4.south) + (-0.2,0)$)	 edge node{} ($(2.north) + (-0.2,0)$)
	;

\path[thick,->,>=stealth, dotted, line width=1pt]
	($(1.east) + (0,0.1) $) edge node{} ($(2.west) + (0,0.1)$)
	($(2.west) + (0,-0.1) $) edge node{} ($(1.east) + (0,-0.1)$)
	;

\node[] at ($0.5*(1.north) + 0.5*(3.south) + (0.5,0)$) {$u_t^1$};
\node[] at ($0.5*(1.north) + 0.5*(3.south) + (-0.5,0)$) {$x_t^1$};

\node[] at ($0.5*(2.north) + 0.5*(4.south) + (0.5,0)$) {$u_t^2$};
\node[] at ($0.5*(2.north) + 0.5*(4.south) + (-0.5,0)$) {$x_t^2$};

\node[] at ($0.5*(1.east) + 0.5*(2.west) + (0,0.4)$) {$\gamma^1$};
\node[] at ($0.5*(1.east) + 0.5*(2.west) + (0,-0.4)$) {$\gamma^2$};

\end{tikzpicture}
\caption{Two-agent system model. Solid lines indicate communication links, dashed lines indicate control links, and dotted lines indicate the possibility of information sharing.}
\label{fig:SystemModel}
\end{center}
\end{figure}

\textbf{Notation.} The collection of matrices $A^1$ and $A^2$ (resp. vectors $x^1$ and $x^2$) is denoted as $A^{1,2}$ (resp. $x^{1,2}$). Given column vectors $x^1$ and $x^2$, the notation $\vecc(x^1, x^2)$ is used to denote the column vector formed by stacking  $x^1$ on top of  $x^2$. We use $[P^{\cdot, \cdot}]_{1,2}$ and $\textbf{diag}(P^1,P^2)$ to denote the following block matrices, $[P^{\cdot, \cdot}]_{1,2} := \begin{bmatrix}
P^{11} & P^{12}\\ P^{21} & P^{22} \end{bmatrix}$, $\textbf{diag}(P^1,P^2) = 
\begin{bmatrix}
P^{1} & \mathbf{0} \\ \mathbf{0}& P^{2} \end{bmatrix}$.
%\begin{align*}
%[P^{\cdot, \cdot}]_{1,2} := \begin{bmatrix}
%P^{11} & P^{12}\\ P^{21} & P^{22} \end{bmatrix}, \quad 
%\textbf{diag}(P^1,P^2) = 
%\begin{bmatrix}
%P^{1} & \mathbf{0} \\ \mathbf{0}& P^{2} \end{bmatrix}.
%\end{align*}

\section{Problem Formulation}
\label{problem_formulation}
Consider a multi-agent Linear-Quadratic (LQ) system consisting of two systems and two associated agents as shown in Figure \ref{fig:SystemModel}. The linear dynamics of systems 1 and 2 are given by
\begin{align}
x_{t+1}^1 &= A_*^{1} x_t^1 + B_*^{1} u_t^1 + w_t^1,
\notag \\
x_{t+1}^2 &= A_*^{2} x_t^2 + B_*^{2} u_t^2 + w_t^2,
\label{Model:dynamics}
\end{align}
where for $n \in \{1,2\}$, $x_t^n \in \R^{d_x^n}$ is the state of system $n$ and $u_t^n \in \R^{d_u^n}$ is the action of agent $n$. $A_*^{1,2}$ and $B_*^{1,2}$ are system matrices with appropriate dimensions. We assume that for $n \in \{1,2\}$, $w_t^n$, $t \geq0$, are i.i.d with standard Gaussian distribution $\mathcal{N}(\mathbf{0}, \mathbf{I})$. The initial states $x_0^{1,2}$ are assumed to be fixed and known.

The overall system dynamics can be written as,
\begin{align}
&x_{t+1} = A_* x_t + B_* u_t + w_t,
\label{Model:system_overall}
\end{align}
where we have defined $x_t = \vecc(x_t^1,x_t^2), u_t = \vecc(u_t^1,u_t^2), w_t = \vecc(w_t^1, w_t^2)$, $A_* = \textbf{diag}(A_*^1,A_*^2)$, and $B_* = \textbf{diag}(B_*^1,B_*^2)$.

% \begin{bmatrix}
%A_*^{1} & \mathbf{0} \\ \mathbf{0} & A_*^{2}
%\end{bmatrix}$, and $B_* = \begin{bmatrix}
%B_*^{1} & \mathbf{0} \\ \mathbf{0} & B_*^{2}
%\end{bmatrix}$.
At each time $t$, agent $n$, $n \in \{1,2\}$, perfectly observes the state $x_t^n$ of its respective system. The pattern of  information sharing plays an important role in the analysis of multi-agent systems. In order to capture different information sharing patterns between the agents, let $\gamma^n \in \{0,1\}$ be a fixed binary variable indicating the availability of a communication link from agent $n$ to the other agent. Then, $i_t^n$ which is the information sent by agent $n$ to the other agent can be written as,
\begin{align}
i_t^n = \begin{cases}
    x_t^n     & \quad \text{if } \gamma^n=1    \\
        \varnothing   & \quad \text{otherwise}
  \end{cases}.
\end{align}
At each time $t$, agent $n$'s action is a function $\pi_t^n$ of its information $h_t^n$, that is, $u_t^n = \pi_t^n(h_t^n)$ where $h_t^1 = \{x_{0:t}^1, u_{0:t-1}^1, i_{0:t}^{2}\}$ and $h_t^2 = \{x_{0:t}^2, u_{0:t-1}^2, i_{0:t}^{1}\}$. Let $\pi = (\pi^1, \pi^2)$ where $\pi^n = (\pi_0^n, \pi_1^n, \ldots)$. We will look at two following information sharing patterns:\footnote{The other possible pattern is two-way information sharing ($\gamma^1=\gamma^2=1$). In this case, both agents observe the states of both systems. Due to the lack of space, we delegate this case to Appendix \ref{sec:two_way}.}
\begin{enumerate}
\item No information sharing ($\gamma^1=\gamma^2=0$),
\item One-way information sharing from agent 1 to agent 2 ($\gamma^1=1$, $\gamma^2=0$).
\end{enumerate}

At time $t$, the system incurs an instantaneous cost $c(x_t,u_t)$, which is a quadratic function given by
\begin{align}
c(x_t,u_t) &=  x_t^\tp Q x_t + u_t^\tp R u_t, 
\label{Model:cost}
\end{align}
where $Q = [Q^{\cdot, \cdot}]_{1,2}$ is a known symmetric positive semi-definite (PSD) matrix and $R = [R^{\cdot, \cdot}]_{1,2}$ is a known symmetric positive definite (PD) matrix. 

%$Q = \begin{bmatrix}
%Q^{11} & Q^{12}\\ Q^{21} & Q^{22}
%\end{bmatrix}$ is a known symmetric positive semi-definite (PSD) matrix and $R = \begin{bmatrix}
%R^{11} & R^{12}\\R^{21} & R^{22}
%\end{bmatrix}$ is a known symmetric positive definite (PD) matrix. 
\subsection{The optimal multi-agent linear-quadratic problem}
\label{sec:optimal_dec_LQ}
Let $\theta_*^n = [A_*^n,B_*^n]$ be the dynamics parameter of system $n$, $n \in \{1,2\}$. % and  $\theta_* = (\theta_*^1, \theta_*^2)$. 
When $\theta_*^1$ and $\theta_*^2$ are perfectly known to the agents, minimizing the infinite horizon average cost  is a multi-agent stochastic Linear Quadratic (LQ) control problem. Let $J(\theta_*^{1,2})$ be the optimal infinite horizon average cost under $\theta_*^{1,2}$, that is,
%\vspace{-10mm}
\begin{align}
J(\theta_*^{1,2}) = \inf_{\pi} \limsup_{T\rightarrow\infty} \frac{1}{T} \sum_{t=0}^{T-1} \ee^{\pi}[c(x_t,u_t) | \theta_*^{1,2}].
\label{eq:optimal_cost_decentralized}
\end{align}
%Due to stability issues in the infinite horizon problem, we make the following  standard assumption on the system and cost matrices.
We make the following standard assumption about the multi-agent stochastic LQ problem.
\begin{assumption}
\label{assum:det_stb}
$(A_*,B_*)$ is stabilizable\footnote{$(A_*, B_*)$ is stabilizable if there exists a gain matrix $K$ such that $A_*+B_*K$ is stable.} and $(A_*,Q^{1/2})$ is detectable\footnote{$(A_*, Q^{1/2})$ is detectable if there exists a gain matrix $H$ such that $A_*+HQ^{1/2}$ is stable.}.
\end{assumption}
The above decentralized stochastic LQ problem has been studied by \citet{ouyang2018optimal}. The following lemma summarizes this result.
\begin{lemma}[\citet{ouyang2018optimal}]
\label{lm:opt_strategies}
Under Assumption \ref{assum:det_stb}, the optimal control strategies are given by
\begin{align}
u^{1}_t &= K^1(\theta_*^{1,2})  \begin{bmatrix}
\hat x_t^1 \\
\hat x_t^2 
\end{bmatrix}+ \tilde K^1 (\theta_*^{1}) (x_t^1 - \hat x_t^1),
\notag \\ 
u^{2}_t &= K^2(\theta_*^{1,2})  \begin{bmatrix}
\hat x_t^1 \\
\hat x_t^2 
\end{bmatrix} + \tilde K^2 (\theta_*^{2}) (x_t^2 - \hat x_t^2),
 \label{eq:opt_U}
\end{align}
{where the gain matrices $
K^1(\theta_*^{1,2}), K^2(\theta_*^{1,2}), \tilde K^1(\theta_*^{1})$, and $\tilde K^2(\theta_*^{2})$ can be computed offline\footnote{See Appendix \ref{lemma1_complete} for the complete description of this result.} and} $\hat x_t^n$, $n \in \{1,2\}$, can be computed recursively according to
\begin{align}
&\hat x_0^n =  x_0^n, 
\quad
\hat x_{t+1}^n=  \notag \\  
&
\begin{cases}
    x_{t+1}^n     & \text{if } \gamma^n=1    \\
     A_*^{n} \hat x_t^n + B_*^{n} K^{n}(\theta_*^{1,2}) \vecc(\hat x_t^1, \hat x_t^2)   & \text{otherwise}
  \end{cases}.
\label{eq:estimator_t_infinite}
\end{align}

\end{lemma}

\subsection{The multi-agent reinforcement learning problem}
The problem we are interested in is to minimize the infinite horizon average cost when the matrices $A_*$ and $B_*$ of the system are unknown. In this case, the control problem 
\red{described by \eqref{Model:dynamics}-\eqref{Model:cost}}
can be seen as a Multi-Agent Reinforcement Learning (\texttt{MARL}) problem where both agents need to learn the system parameters $\theta_*^1 = [A_*^{1}, B_*^1]$ and $\theta_*^2 = [A_*^{2}, B_*^2]$ in order to minimize the infinite horizon average cost. The learning performance of policy $\pi$ is measured by the cumulative regret over $T$ steps defined as,
\begin{align}
R(T, \pi) = \sum_{t=0}^{T-1} \left[ c(x_t, u_t) - J(\theta_*^{1,2}) \right],
\label{eq:regret_decentralized}
\end{align} 
which is the difference between the performance of the agents under policy $\pi$ and the
optimal infinite horizon cost under full information about the system dynamics. Thus, the regret
can be interpreted as a measure of the cost of not knowing the system dynamics.% The learning objective is to find a decentralized control algorithm that minimizes the expected regret.

%Due to stability issues in the infinite horizon problem, we make the following  standard assumption on the system and cost matrices.
%\begin{assumption}
%\label{assum:det_stb}
%For any $\theta_*^n = [A_*^n,B_*^n]$ in the support of prior distributions $\mu_0^n$, $n \in \{1,2\}$, $(A_*,Q^{1/2})$ is detectable and $(A_*,B_*)$ is stabilizable.%\footnote{$(A, B)$ is stabilizable if there exists a control gain matrix $K$ s.t. $A+BK$ is stable (i.e., all eigenvalues are in $(-1, 1)$).}.
%\end{assumption}

\section{An Auxiliary  Single-Agent LQ Problem}
\label{sec:centralized_LQ}
In this section, we construct an auxiliary single-agent LQ control problem based on the \texttt{MARL} problem of Section \ref{problem_formulation}. 
\red{
This auxiliary single-agent LQ control problem is inspired by the \textit{common information based coordinator} (which has been developed in non-learning settings in \citet{nayyar2013decentralized} and \citet{asghari2018optimal} and the references therein).
We will later use the auxiliary problem as a coordination mechanism for our \texttt{MARL} algorithm.
}

%This auxiliary single-agent LQ control problem will be used later for the regret analysis of the \texttt{MARL} problem.

Consider a single-agent system with dynamics

\vspace{-6mm}
\begin{align}
x_{t+1}^{\diamond} = A_* x_t^{\diamond} + B_* u_t^{\diamond} + \begin{bmatrix}
w_t^1 \\ \mathbf{0}
\end{bmatrix},
\label{Model:system_LQ}
\end{align}
where $x_t^{\diamond} \in \R^{d_x^1 + d_x^2}$ is the state of the system, $u_t^{\diamond} \in  \R^{d_u^1 + d_u^2}$ is the action of the auxiliary agent, $w_t^1$ is the noise vector of system 1 defined in \eqref{Model:dynamics}, and matrices $A_*$ and $B_*$ are as defined in \eqref{Model:system_overall}. The initial state $x_0^{\diamond}$ is assumed to be equal to $x_0$. The action $u_t^{\diamond} = \pi_t^{\diamond} (h_t^{\diamond})$ at time $t$ is a function of the history of observations $h_t^{\diamond} = \{x_{0:t}^{\diamond}, u_{0:t-1}^{\diamond}\}$. The auxiliary agent's strategy is denoted by $\pi^{\diamond} = (\pi_1^{\diamond}, \pi_2^{\diamond},\ldots)$. The instantaneous cost $c(x_t^{\diamond},u_t^{\diamond})$ of the system is a quadratic function given by
\begin{align}
c(x_t^{\diamond}, u_t^{\diamond}) &= 
(x_t^{\diamond})^\tp Q x_t^{\diamond} + (u_t^{\diamond})^\tp R u_t^{\diamond},
\label{Model:cost_LQ}
\end{align}
where matrices $Q$ and $R$ are as defined in \eqref{Model:cost}.

%When $\theta_*^1$ and $\theta_*^2$ are known to the auxiliary agent, minimizing the infinite horizon average cost is a single-agent stochastic Linear-Quadratic (LQ) control problem. Let $J^{\diamond}(\theta_*^{1,2})$ be the optimal infinite horizon average cost under $\theta_*^{1,2}$, that is,
%\begin{align}
%J^{\diamond}(\theta_*^{1,2}) = \inf_{\pi^{\diamond}} \limsup_{T\rightarrow\infty} \frac{1}{T} \sum_{t=0}^{T-1} \ee^{\pi^{\diamond}}[c(x_t^{\diamond}, u_t^{\diamond}) | \theta_*^{1,2}].
%\label{eq:optimal_cost_centralized}
%\end{align}
%%The above centralized stochastic LQ problem has been widely studied in the literature \citep{}. The following lemma summarizes this result.
%Under Assumption \ref{assum:det_stb}, the optimal strategy $\pi^{\diamond *}$ is given by $u_t^{\diamond} = K(\theta_*^{1,2}) x_t^{\diamond}$ \citep{kumar2015stochastic,bertsekas1995dynamic} where $K(\theta_*^{1,2})$ is as defined\footnote{See Appendix \ref{lemma1_complete} for a complete description of this result.} in Lemma \ref{lm:opt_strategies}.

%Then, the following lemma summarizes the result for the optimal infinite horizon single-agent LQ control problem\footnote{See Appendix \ref{lemma1_complete} for a complete description of this result.}.
%\begin{lemma}[\citep{kumar2015stochastic,bertsekas1995dynamic}]
%\label{lm:opt_strategies_centralized}
%Under Assumption \ref{assum:det_stb}, the optimal strategy $\pi^{\diamond *}$ is given by $u_t^{\diamond} = K(\theta_*^{1,2}) x_t^{\diamond}$ where $K(\theta_*^{1,2}) := \begin{bmatrix}
%K^1(\theta_*^{1,2}) \\ K^2(\theta_*^{1,2})
%\end{bmatrix}$.
%\end{lemma}

When the parameters $\theta_*^{1}$  and $\theta_*^{2}$ are unknown, we will have a Single-Agent Reinforcement Learning (\texttt{SARL}) problem. In this problem, the regret of a policy $\pi^{\diamond}$ over $T$ steps is given by
\begin{align}
R^{\diamond}(T, \pi^{\diamond}) =  \sum_{t=0}^{T-1} \left[ c(x_t^{\diamond}, u_t^{\diamond}) - J^{\diamond}(\theta_*^{1,2})  \right],
\label{eq:regret_centralized}
\end{align} 
where $J^{\diamond}(\theta_*^{1,2})$ is the optimal infinite horizon average cost under $\theta_*^{1,2}$.

%\red{
%\begin{remark}
%\label{rem:common_based_estimator}
%The state $x_t^{\diamond}$ of the auxiliary \texttt{SARL} can be interpreted as an estimate of the state $x_t$ of the overall system \eqref{Model:system_overall} that each agent computes based on the common information between them. In fact, the \texttt{SARL} dynamics in \eqref{Model:system_LQ}  can be seen as the filtering equation for this common estimate.
%\end{remark}
%}

Existing algorithms for the \texttt{SARL} problem are generally based on the two following approaches: Optimism in the Face of Uncertainty (OFU) \citep{abbasi2011regret, ibrahimi2012efficient, faradonbeh2017regret, faradonbeh2019applications} and Thompson Sampling (TS) (also known as posterior sampling) \citep{faradonbeh2017regret, abbasi2015bayesian, abeille2018improved}. In spite of the differences among these algorithms, all can be generally described as the \texttt{AL-SARL} algorithm (algorithm for the \texttt{SARL} problem). In this algorithm, at each time $t$, the agent interacts with a \texttt{SARL} learner (see Appendix \ref{sec:details_SARL_learner} for a detailed description the \texttt{SARL} learner) by feeding time $t$ and the state $x_t^{\diamond}$ to it and receiving \textit{estimates} $\theta_t^1 = [A_t^1, B_t^1]$ and $\theta_t^2 = [A_t^2, B_t^2]$ of the unknown parameters $\theta_*^{1,2}$. Then, the agent uses $\theta_t^{1,2}$ to calculate the gain matrix $K(\theta_t^{1,2})$ (see Appendix \ref{lemma1_complete} for a detailed description of this matrix) and executes the action $u_t^{\diamond} = K(\theta_t^{1,2}) x_t^{\diamond}$. As a result, a new state $x_{t+1}^{\diamond}$ is observed.

Among the existing algorithms, OFU-based algorithms of \citet{abbasi2011regret,ibrahimi2012efficient,faradonbeh2017regret, faradonbeh2019applications} and the TS-based algorithm of \citet{faradonbeh2017regret} achieve a $\tilde O (\sqrt{T})$ regret for the \texttt{SARL} problem (Here $\tilde O (\cdot)$ hides constants and logarithmic factors).
\begin{minipage}{.49\textwidth}
\begin{algorithm}[H]
   \caption{\texttt{AL-SARL}} %\citep{ouyang2017learning}}
   \label{alg:TS_centralized}
\begin{algorithmic}
  \State Initialize $\mathcal{L}$ and $x_0^{\diamond}$
   \For{$t=0,1,\ldots$}
\State Feed time $t$ and state $x_t^{\diamond}$ to $\mathcal{L}$ and get $\theta_t^{1}$ and $\theta_t^{2}$%\theta_t^1 = [A_t^1, B_t^1]$ and $\theta_t^2 = [A_t^2, B_t^2]$
   \State Compute $K(\theta_t^{1,2})$
   \State Execute $u_t^{\diamond} = K(\theta_t^{1,2}) x_t^{\diamond}$
  % \Endif
   \State Observe new state $x_{t+1}^{\diamond}$
  % \State Update $\mathcal{L}$
   \EndFor
\end{algorithmic}
\end{algorithm}
\end{minipage} 

\begin{minipage}{.5\textwidth}
\begin{center}
\begin{tikzpicture}
\node [rectangle,draw,minimum width=1.2cm,minimum height=0.75cm,line width=1pt,rounded corners, fill=black!25]at (-1.0,0) (1) {
\begin{small}
\begin{tabular}{c}
$\mathcal{L}$
\end{tabular}
\end{small}}; 
\path[thick,->,>=stealth, dashed, line width=1pt]
	($(1.north) + (0,0.5)$) edge node{} ($(1.north) + (0,0)$)
	%($(1.south) + (0,-0.5)$) edge node{} ($(1.south) + (0,0)$)
	;
\path[thick,->,>=stealth,line width=1pt]
	($(1.west) + (-0.5,0)$) edge node{} ($(1.west) + (0,0)$)
	($(1.east) + (0,0)$)	 edge node{} ($(1.east) + (0.5,0)$)
	;

\node[] at ($(1.north) + (0,0.7)$) {Initialize parameters};
%\node[] at ($(1.south) + (0,-0.7)$) {Update};
\node[] at ($(1.west) + (-1.2,0.25)$) {state $x_t^{\diamond}$};
\node[] at ($(1.west) + (-1.2,-0.25)$) {time $t$};
\node[] at ($(1.east) + (1.5,0.25)$) {$\theta_t^1 = [A_t^1, B_t^1]$};
\node[] at ($(1.east) + (1.5,-0.25)$) {$\theta_t^2 = [A_t^2, B_t^2]$};

\end{tikzpicture}
%\caption{\texttt{SARL} learner as a block box.}
%\label{fig:SystemModel}
\end{center}
\end{minipage}
%\end{tabular}
%\end{table}

\section{Main Results}
In this section, we start with the regret analysis for the case where the parameters of both systems are unknown (that is, $\theta_*^1$ and $\theta_*^2$ are unknown) and there is no information sharing between the agents (that is, $\gamma^1=\gamma^2=0$). The detailed proofs for all results are in the appendix.

\subsection{Unknown $\theta_*^1$ and $\theta_*^2$, no information sharing ($\gamma^1=\gamma^2=0$)}
\label{sec:main_results_1}
For the \texttt{MARL} problem of this section (it is called \texttt{MARL1} for future reference), we show that there is no learning algorithm with a sub-linear in $T$ regret for all instances of  the \texttt{MARL1} problem. The following theorem states this result. 

\begin{theorem}
\label{thm:negative_results_case1}
There is no algorithm that can achieve a lower-bound better than $\Omega(T)$ on the regret of all instances of the \texttt{MARL1} problem.
\end{theorem}
A $\Omega(T)$ regret implies that the average performance of the learning algorithm has at least a constant gap from the ideal performance of informed agents. This prevent efficient learning performance even in the limit.
Theorem \ref{thm:negative_results_case1} implies that in a \texttt{MARL1} problem where the system dynamics are unknown, learning is not possible without communication between the agents. The proof of Theorem \ref{thm:negative_results_case1} also provides the following result.

\begin{lemma}
Consider a \texttt{MARL} problem where the parameter of system 2 (that is, $\theta_*^2$) is known to both agents and only the parameter of system 1 (that is, $\theta_*^1$) is unknown. Further, there is no communication between the agents. Then, there is no algorithm that can achieve a lower-bound better than $\Omega(T)$ on the regret of all instances of this \texttt{MARL} problem.
\end{lemma}

The above results imply that if the parameter of a system is unknown, both agents should receive information about this unknown system; otherwise, there is no learning policy $\pi$ that can guarantee a sub-linear in $T$ regret for all instances of this \texttt{MARL} problem.

%Theorem \ref{thm:negative_results_case1} implies that in a \texttt{MARL1} problem, when the parameter of a system is unknown, both agents should receive observation about the state of this system; otherwise, there is no learning policy $\pi$ that can guarantee a lower-than-$T$ regret for all instances of the \texttt{MARL1} problem. For instance, if the parameter of system 2 is unknown (that is, $\theta_*^2$ is unknown), without having any observation about the state $x_t^2$ of this plant, there is no way for agent 1 to learn about $\theta_*^2$.

In the next section, we assume that $\theta_*^2$ is known to both agents and only $\theta_*^1$ is unknown. Further, we assume the presence of a communication link from agent 1 to agent 2, that is, $\gamma^1=1$. This communication link allows agent 2 to receive feedback about the state $x_t^1$ of system 1 and hence, remedies the impossibility of learning for agent 2.

\subsection{Unknown $\theta_*^1$, one-way information sharing from agent 1 to agent 2 ($\gamma^1=1$, $\gamma^2=0$)}
\label{sec:main_results_2}
%For the \texttt{MARL} of this section (it is called \texttt{MARL2} for future reference), we propose the \texttt{AL-MARL} algorithm based on the \texttt{AL-SARL} algorithm.  \texttt{AL-MARL} algorithm is a multi-agent algorithm which is performed independently by the agents. In the \texttt{AL-MARL} algorithm, each agent has its own learner $\mathcal{L}$ and uses it to learn the unknown parameter $\theta_*^1$ of system 1. In this algorithm, $\check x_{t}^2$ (described in the \texttt{AL-MARL} algorithm) is a proxy for $\hat x_{t}^2$ of \eqref{eq:estimator_t_infinite} where instead of the unknown parameter $\theta_*^1$, we have $\theta_t^1$.
In this section, we consider the case where only system 1 is unknown and there is one-way communication from agent 1  to agent 2. Despite this one-way information sharing, the two agents still have different information. In particular, at each time agent 2 observes the state $x^2_t$ of system 2 which is {not available} to agent 1.
For the \texttt{MARL} of this section (it is called \texttt{MARL2} for future reference), we propose the \texttt{AL-MARL} algorithm which builds on the auxiliary \texttt{SARL} problem of Section \ref{sec:centralized_LQ}.
\texttt{AL-MARL} algorithm is a decentralized multi-agent algorithm which is performed independently by the agents. Every agent independently constructs an auxiliary \texttt{SARL} problem where $x_t^{\diamond} = \vecc(x_t^1, \check x_t^2)$
and applies an \texttt{AL-SARL} algorithm with its own learner $\mathcal{L}$ to it in order to learn the unknown parameter $\theta_*^1$ of system 1. In this algorithm, $\check x_{t}^2$ (described in the \texttt{AL-MARL} algorithm) is a proxy for $\hat x_{t}^2$ of \eqref{eq:estimator_t_infinite} updated using the estimate $\theta_t^1$  instead of the unknown parameter $\theta_*^1$.

%\blue{
%For the \texttt{MARL} of this section (it is called \texttt{MARL2} for future reference), we propose the \texttt{AL-MARL} algorithm which builds on the auxiliary \texttt{SARL} problem of Section \ref{sec:centralized_LQ}.
%\texttt{AL-MARL} algorithm is a (decentralized?) multi-agent algorithm which is performed independently by the agents. Every agent independently constructs an auxiliary \texttt{SARL} problem where $x_t^{\diamond} = \vecc(x_t^1, \check x_t^2)$
%and applies an \texttt{AL-SARL} algorithm with its own learner $\mathcal{L}$ to it in order to learn the unknown parameter $\theta_*^1$ of system 1. In this algorithm, $\check x_{t}^2$ (described in the \texttt{AL-MARL} algorithm) is a proxy for $\hat x_{t}^2$ of \eqref{eq:estimator_t_infinite} where instead of the unknown parameter $\theta_*^1$, we have $\theta_t^1$.}

At time $t$, each agent feeds $\vecc(x_t^1, \check x_t^2)$ to its own \texttt{SARL} learner $\mathcal{L}$ and gets $\theta_t^1$ and $\theta_t^2$. Note that both agents already know the true parameter $\theta_*^2$, hence they only use $\theta_t^1$ to compute their gain matrix {$K^{\texttt{agent_ID}}(\theta_t^1, \theta_*^2)$} and use this gain matrix to compute their actions $u_t^1$ and $u_t^2$ according to  the \texttt{AL-MARL} algorithm.
Note that agent 2 needs $\tilde K^2(\theta_*^2)$ to calculate its actions $u_t^2$. However, we know that $\tilde K^2(\theta_*^2)$ is independent of the unknown parameter $\theta_*^1$ and hence, $\tilde K^2(\theta_*^2)$ can be calculated prior to the beginning of the algorithm. After the execution of the actions $u_t^{1}$ and $u_t^{2}$ by the agents, both agents observe the new state $x_{t+1}^1$ and agent 2 further observes the new state $x_{t+1}^2$. Finally, each agent independently computes $\check x_{t+1}^2$.

%{\color{blue}
%\begin{remark}
%Note that $\check x_{t+1}^2$ in the \texttt{AL-MARL} algorithm is a proxy for $\hat x_{t+1}^2$ of Lemma \ref{lm:opt_strategies} where instead of the unknown parameter $\theta_*^1$, we have $\theta_t^1$.
%\end{remark} }

\red{
\begin{remark}
\label{rem:common_based_estimator}
The state $x_t^{\diamond}$ of the auxiliary \texttt{SARL} can be interpreted as an estimate of the state $x_t$ of the overall system \eqref{Model:system_overall} that each agent computes based on the common information between them. In fact, the \texttt{SARL} dynamics in \eqref{Model:system_LQ}  can be seen as the filtering equation for this common estimate.
\end{remark}
}

\red{
\begin{remark}
We want to emphasize that unlike the idea of \textit{centralized training with decentralized execution} \citep{foerster2016learning, dibangoye2018learning, hernandez2018multiagent}, the \texttt{AL-MARL} algorithm is an \textit{online decentralized learning} algorithm. This means that there is no centralized learning phase in the setup where agents can collect information or have access to a simulator. The agents are simultaneously learning and controlling the system.
%\textit{centralized training} approach, 
\end{remark}
}

\begin{remark}
\label{rem:different_samples}
Since the \texttt{SARL} learner $\mathcal{L}$ can include taking samples and solving optimization problems, due to the independent execution of the \texttt{AL-MARL} algorithm, agents might receive different $\theta_t^{1,2}$ from their own learner $\mathcal{L}$.
\end{remark}

In order to avoid the issue pointed out in Remark \ref{rem:different_samples}, we make an assumption about the output of the \texttt{SARL} learner $\mathcal{L}$.

\begin{assumption}
\label{assum:seed}
Given the same time and same state input to the \texttt{SARL} learner $\mathcal{L}$, the outputs $\theta_t^{1,2}$ from different learners $\mathcal{L}$ are the same.
\end{assumption}
Note that Assumption \ref{assum:seed} can be easily achieved by setting the same initial sampling seed (if the \texttt{SARL} learner $\mathcal{L}$ includes taking samples) or by setting the same tie-breaking rule among possible similar solutions of an optimization problem (if the \texttt{SARL} learner $\mathcal{L}$ include solving optimization problems). Now, we present the following result which is based on Assumption \ref{assum:seed}. 

\begin{theorem}
\label{thm:regret_bound_case2}
Under Assumption \ref{assum:seed}, let $R(T, \texttt{AL-MARL})$ be the regret for the \texttt{MARL2} problem under the policy of the \texttt{AL-MARL} algorithm and $R^{\diamond}(T, \texttt{AL-SARL})$ be the regret for the auxiliary \texttt{SARL} problem under the policy of the \texttt{AL-SARL} algorithm. Then for any $\delta \in (0, 1/e)$, with probability at least $1 - \delta$, 
\begin{align}
R(T, \texttt{AL-MARL})  \leq R^{\diamond}(T, \texttt{AL-SARL}) +\log(\frac{1}{\delta}) \tilde K \sqrt{T}.
\end{align} 
\end{theorem}

\begin{algorithm}[tb]
   \caption{\texttt{AL-MARL}} %\citep{ouyang2017learning}}
   \label{alg:TS_decentralized}
\begin{algorithmic}
\State \textbf{Input}: \texttt{agent_ID}, learner $\mathcal{L}$, $x_0^1$, and $x_0^2$
  \State Initialize $\mathcal{L}$ and $\check x_0^2 = x_0^2$
   \For{$t=0,1,\ldots$}
\State Feed time $t$ and state $\vecc(x_t^1, \check x_t^2)$ to $\mathcal{L}$ and 
\State \hspace{2mm} get $\theta_t^1 = [A_t^1, B_t^1]$ and $\theta_t^2 = [A_t^2, B_t^2]$
 \State Compute $K^{\texttt{agent_ID}}(\theta_t^1, \theta_*^2)$ %from \eqref{eq:K_infinite}
   \If{$\texttt{agent_ID}=1$}
   \State Execute $u_t^1= K^1(\theta_t^1, \theta_*^2) \vecc(x_t^1, \check x_t^2)$
   \Else
  \State Execute $u_t^2= K^2(\theta_t^1, \theta_*^2) \vecc(x_t^1, \check x_t^2)$
  \State \hspace{18mm} $+ \tilde K^2(\theta_*^2) (x_t^2 - \check x_t^2)$ 
   \EndIf
   \State  Observe new state $x_{t+1}^1$
   \State Compute $\check x_{t+1}^2 = A_*^2 \check x_t^2$
   \State \hspace{22mm}  $+ B_*^2 K^2(\theta_t^1, \theta_*^2) \vecc(x_t^1, \check x_t^2)  $
     \If{$\texttt{agent_ID}=2$}
     \State Observe new state $x_{t+1}^2$
     \EndIf
    % \State Update $\mathcal{L}$
%     \State Compute $\check x_{t+1}^2 = A_*^2 \check x_t^2 + B_*^2 K^2(\theta_t^1, \theta_*^2) \vecc(x_t^1, \check x_t^2)  $
   \EndFor
\end{algorithmic}
\end{algorithm}

This result shows that under the policy of the \texttt{AL-MARL} algorithm, the regret for the \texttt{MARL2} problem is upper-bounded by the regret for the auxiliary \texttt{SARL} problem constructed in Section \ref{sec:centralized_LQ} under the policy of the \texttt{AL-SARL} algorithm plus a term bounded by $O(\sqrt{T})$. %This theorem further implies that expected regret bounds for \texttt{AL-SARL} algorithm also hold for the \texttt{AL-MARL} algorithm.

\begin{corollary}
\label{cor:regret_bound_case}
\texttt{AL-MARL} algorithm with the OFU-based \texttt{SARL} learner $\mathcal{L}$ of \citet{abbasi2011regret,ibrahimi2012efficient,faradonbeh2017regret, faradonbeh2019applications} or the TS-based \texttt{SARL} learner $\mathcal{L}$ of \citet{faradonbeh2017regret} achieves a $\tilde O(\sqrt{T})$ regret for the \texttt{MARL2} problem.
\end{corollary}
%{\color{blue}
%\begin{remark}
%%\label{rem:different_samples}
%Despite the one-way information sharing, the two agents still have different information. In particular, at each time agent 2 observes the state $x^2_t$ of system 2 which is unknown to agent 1.
%\end{remark}
%}

\red{
\begin{remark}
The idea of constructing a centralized problem for \texttt{MARL} is similar in spirit to the centralized algorithm perspective adopted in \citet{dibangoye2018learning}.
However, we would like to emphasize that the auxiliary \texttt{SARL}  problem is different from the centralized oMDP in \citet{dibangoye2018learning}. The oMDP is a \textbf{deterministic }MDP with no observations of the belief state. Our single agent problem is inspired by the \textit{common information based coordinator}  developed in non-learning settings in \citet{nayyar2013decentralized} and \citet{asghari2018optimal}.
The difference from oMDP is reflected in the fact that the state evolution in the \texttt{SARL} is \textbf{stochastic} (see \eqref{Model:system_LQ}). 
As discussed in Remark \ref{rem:common_based_estimator}, the state of the auxiliary \texttt{SARL} can be interpreted as the common information based state estimate. In our \texttt{AL-MARL} algorithm, both agents use this randomly evolving, common information based state estimate to learn the unknown parameters in an identical manner. This removes the potential mis-coordination among agents due to difference in information and allows for efficient learning. 
%The state of the auxiliary \texttt{SARL}  can be interpreted as an estimate of the state $x_t$ of the overall system  \eqref{Model:system_overall} that each agent computes based on the common information between them. In fact, the SARL dynamics in \eqref{Model:system_LQ} can be seen as the filtering equation for this common estimate. In our \texttt{MARL} algorithm, both agents use this randomly evolving, common information based state estimate to learn the unknown parameters in an identical manner. This removes the potential mis-coordination among agents due to difference in information and allows for efficient learning. 
\end{remark}
}

\subsection{Extension to \texttt{MARL} problems with more than 2 systems and 2 agents}
While the results of Sections \ref{sec:main_results_1} and \ref{sec:main_results_2} are for \texttt{MARL} problems with 2 systems and 2 agents, these results can be extended to \texttt{MARL} problems with an arbitrary number $N$ of agents and systems in the following sense.

\begin{lemma}
Consider a \texttt{MARL} problem with $N$  agents and systems ($N \geq 2$). Suppose there is a system $n$ and an agent $m$, $m \neq n$, such that system $n$ is unknown and there is no communication  from agent $n$ to agent $m$. Then, there is no algorithm that can achieve a lower-bound better than $\Omega(T)$ on the regret of all instances of this \texttt{MARL} problem.
% If system $n$ is unknown and there is an agent $m$ such that no communication link exists from agent $n$ to agent $m$, then there is no algorithm that can achieve a lower-bound better than $\Omega(T)$ on the regret of all instances of this \texttt{MARL} problem.
\end{lemma}
The above lemma follows from the proof of Theorem \ref{thm:negative_results_case1}.

\begin{theorem}
\label{lem:regret_bound_extension}
Consider a \texttt{MARL} problem with  $N$  agents and systems ($N \geq 2$) where the first $N_1$ systems are unknown and the rest $N - N_1$ systems are known. Further, for any $1 \leq i \leq N_1$, there is  communication  from agent $i$ to all other agents and for any $N_1 +1 \leq j \leq N$, there is no communication from agent $j$ to any other agent. 
Then, there is a learning algorithm that achieves a $\tilde O(\sqrt{T})$ regret for this \texttt{MARL} problem.
%such that the regret of the \texttt{MARL} problem is upper bounded by the regret of the auxiliary \texttt{SARL} problem plus a overhead bounded by $O(\sqrt{T})$.
\end{theorem}
The proof of above theorem requires constructing an auxiliary \texttt{SARL} problem and following the same steps as in the proof of Theorem \ref{thm:regret_bound_case2}.

%Next, we consider the case where both systems are unknown, that is, $\theta_*^{1,2}$ are unknown, however, there are communication links from agent 1 to agent 2 and vice-versa, that is, $\gamma^1=\gamma^2=1$. As we will show in the next section, the availability of a two-way information sharing between the agents allows both agents to learn about both $\theta_*^{1,2}$.

%\section{Theoretical Analysis}

\section{Key Steps in the Proof of Theorem \ref{thm:regret_bound_case2}}
\label{sec:proof_theorem_2}

\subsection*{Step 1: Showing the connection between the auxiliary \texttt{SARL} problem and the \texttt{MARL2} problem}

First, we present the following lemma that connects the optimal infinite horizon average cost $J^{\diamond}(\theta_*^{1,2})$ of the auxiliary \texttt{SARL} problem when $\theta_*^{1,2}$ are known (that is, the auxiliary single-agent LQ problem of Section \ref{sec:centralized_LQ}) and the optimal infinite horizon average cost $J(\theta_*^{1,2})$ of the \texttt{MARL2} problem when $\theta_*^{1,2}$ are known  (that is, the multi-agent LQ problem of Section \ref{sec:optimal_dec_LQ}).
\begin{lemma}
\label{lem:connection_optimal_cost}
$J(\theta_*^{1,2}) = J^{\diamond}(\theta_*^{1,2}) + \tr(D \Sigma),$ where we have defined $D:= Q^{22} + (\tilde K^2(\theta_*^2))^{\tp} R^{22} \tilde K^2(\theta_*^2)$ and $\Sigma$ is as defined in Lemma \ref{lem:Sigma_convergence} in the appendix.
%Let $J^{\diamond}(\theta_*^{1,2})$ be the optimal infinite horizon cost of the auxiliary \texttt{SARL} problem, $J(\theta_*^{1,2})$ be the optimal infinite horizon cost of the \texttt{MARL2} problem, and $\Sigma$ be as defined in Lemma \ref{lem:Sigma_convergence}. Then,
%\begin{align}
%J(\theta_*^{1,2}) = J^{\diamond}(\theta_*^{1,2}) + \tr(D \Sigma),
%\label{eq:optimal_costs_relation}
%\end{align}
%where we have defined $D:= Q^{22} + (\tilde K^2(\theta_*^2))^{\tp} R^{22} \tilde K^2(\theta_*^2)$.
\end{lemma}
%
%\begin{proof}
%See Appendix \ref{proof_lem:connection_optimal_cost} for a proof.
%\end{proof}
%
Next, we provide the following lemma that shows the connection between the cost $c(x_t,u_t)$ in the \texttt{MARL2} problem under the policy of the \texttt{AL-MARL} algorithm and the cost $
c(x_t^{\diamond}, u_t^{\diamond})$ in the auxiliary \texttt{SARL} problem under the policy of the \texttt{AL-SARL} algorithm.
%{\color{blue}
\begin{lemma}
\label{lem:expectation_equalities}
$c(x_t,u_t) \vert_{\texttt{AL-MARL}} = c(x_t^{\diamond}, u_t^{\diamond})\vert_{\texttt{AL-SARL}} + e_t^{\tp} D e_t,$ where $e_t = x_t^2 - \check x_t^2$ and $D$ is as defined in Lemma \ref{lem:connection_optimal_cost}.
%$D= Q^{22} + (\tilde K^2(\theta_*^2))^{\tp} R^{22} \tilde K^2(\theta_*^2)$.
%At each time $t$, the following equality holds between the cost under the policies of the \texttt{AL-SARL} and the \texttt{AL-MARL} algorithms,
%\begin{align}
%c(x_t,u_t) \vert_{\texttt{AL-MARL}} &= c(x_t^{\diamond}, u_t^{\diamond})\vert_{\texttt{AL-SARL}} + e_t^{\tp} D e_t,
%\label{eq:expected_cost_cent_dec}
%\end{align}
%where $e_t = x_t^2 - \check x_t^2$ and $D= Q^{22} + (\tilde K^2(\theta_*^2))^{\tp} R^{22} \tilde K^2(\theta_*^2)$.
\end{lemma}
%}
%
%\begin{proof}
%See Appendix \ref{proof_lem:expectation_equalities} for a proof.
%\end{proof}
%
\subsection*{Step 2: Using the \texttt{SARL} problem to bound the regret of the \texttt{MARL2} problem}
In this step, we use the connection between the auxiliary \texttt{SARL} problem  and our \texttt{MARL2} problem, which was established in Step 1, to prove Theorem \ref{thm:regret_bound_case2}. Note that from the definition of the regret in the \texttt{MARL} problem given by \eqref{eq:regret_decentralized}, we have,
\begin{align}
&R(T, \texttt{AL-MARL}) =  \sum_{t=0}^{T-1} \left[ c(x_t,u_t) \vert_{\texttt{AL-MARL}} - J(\theta_*^{1,2})   \right] \notag \\
%& =  \sum_{t=0}^{T-1}  \left[c(x_t^{\diamond}, u_t^{\diamond}) \vert_{\texttt{AL-SARL}} + e_t^{\tp} D e_t \right]  -    \sum_{t=0}^{T-1} \left[ J^{\diamond}(\theta_*^{1,2}) + \tr(D \Sigma) \right] \notag \\
& = \sum_{t=0}^{T-1}  \left[c(x_t^{\diamond}, u_t^{\diamond}) \vert_{\texttt{AL-SARL}} - J^{\diamond}(\theta_*^{1,2}) \right]  
\notag \\
&+  \sum_{t=0}^{T-1} \left[ e_t^{\tp} D e_t - \tr(D \Sigma) \right] 
\notag \\
&  \leq R^{\diamond}(T, \texttt{AL-SARL}) +\log(\frac{1}{\delta}) \tilde K \sqrt{T},
 \label{eq:regret_inequality_case2}
\end{align} 
%{\color{blue}
where the second equality is correct because of  Lemma \ref{lem:connection_optimal_cost} and Lemma \ref{lem:expectation_equalities}. Further, the last inequality is correct because of the definition of the regret in the the \texttt{SARL} problem given by \eqref{eq:regret_centralized} and the fact that $\sum_{t=0}^{T-1} \left[ e_t^{\tp} D e_t - \tr(D \Sigma) \right] $ is bounded by $\log(\frac{1}{\delta}) \tilde K \sqrt{T}$ from Lemma \ref{thm:prob_regret} in the appendix. 
%}

\section{Experiments}
\label{sec:experiments_main}
In this section, we illustrate the performance of the \texttt{AL-MARL} algorithm through numerical experiments. Our proposed algorithm requires a \texttt{SARL} learner. As the TS-based algorithm of \citet{faradonbeh2017regret} achieves a $\tilde O(\sqrt{T})$ regret for a \texttt{SARL} problem, we use the \texttt{SARL} learner of this algorithm (The details for this \texttt{SARL} learner are presented in Appendix \ref{sec:details_SARL_learner}).
\begin{figure}
%\begin{wrapfigure}{r}{6cm}
 \begin{center}
\includegraphics[width=8cm,height=6cm]{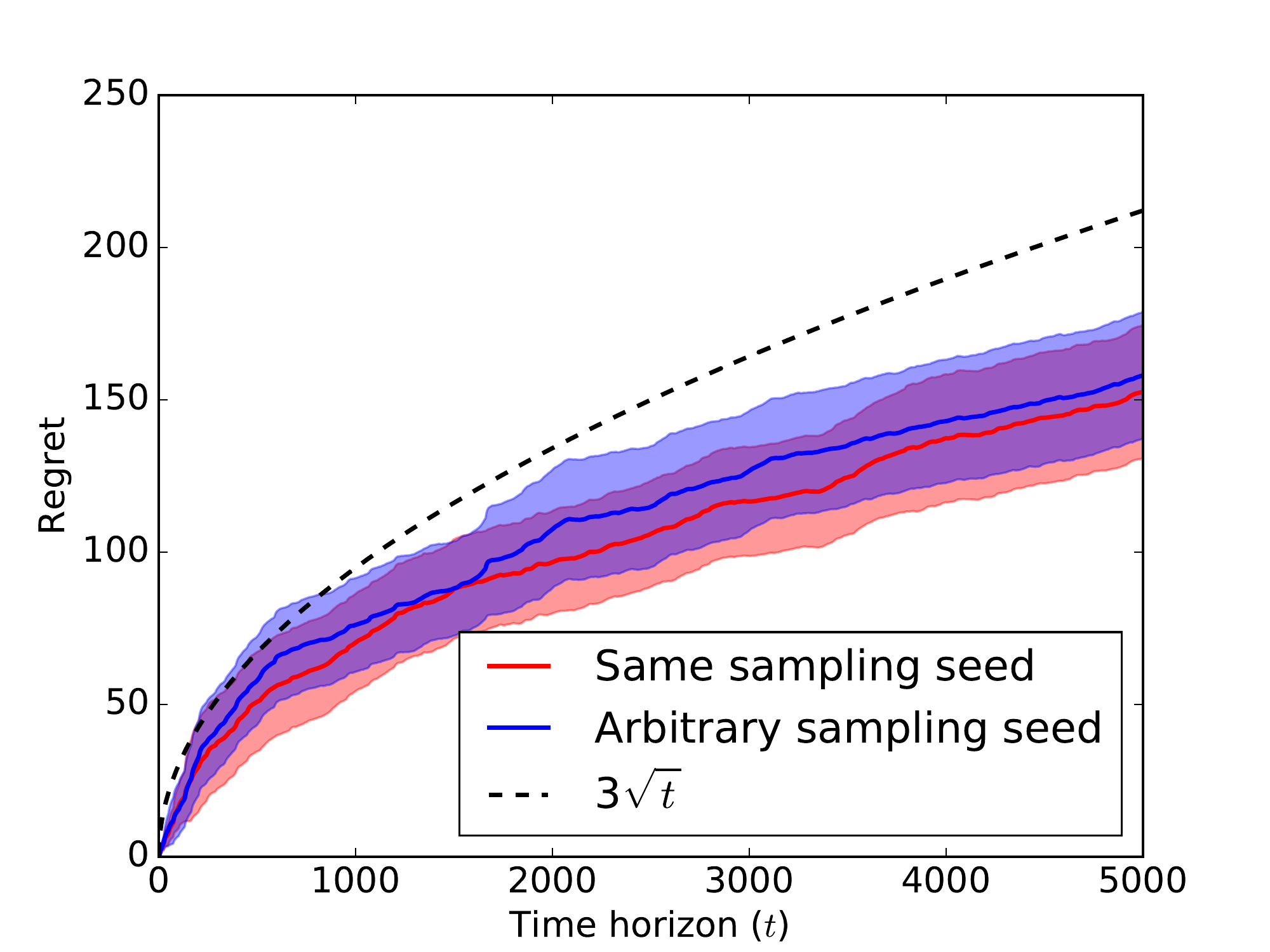}
\end{center}
\caption{\texttt{AL-MARL} algorithm with the \texttt{SARL} learner of \citet{faradonbeh2017regret}}
\label{fig:MARL_1}
%\vspace{-40pt}
%\end{wrapfigure}
\end{figure}

We consider an instance of the \texttt{MARL2} problem (See Appendix \ref{sec:experiments_appendix} for the details). The theoretical result of Theorem \ref{thm:regret_bound_case2} holds when Assumption \ref{assum:seed} is true. Since we use the TS-based learner of \citet{faradonbeh2017regret}, this assumption can be satisfied by setting the same sampling seed between the agents. Here, we consider both cases of same sampling seed and arbitrary sampling seed for the experiments. We ran 100 simulations and show the mean of regret with the 95\% confidence interval for each scenario.

As it can be seen from Figure \ref{fig:MARL_1}, for both of theses cases, our proposed algorithm with the TS-based learner $\mathcal{L}$ of \citet{faradonbeh2017regret} achieves a $\tilde O(\sqrt{T})$ regret for our \texttt{MARL2} problem, which matches the theoretical results of Corollary \ref{cor:regret_bound_case}.

\section{Conclusion}

\red{
In this paper, we tackled the challenging problem of regret analysis in Multi-Agent Reinforcement Learning (MARL). We attempted to solve this challenge in the context of online decentralized learning in multi-agent linear-quadratic (LQ) dynamical systems. First, we showed that if a system is unknown, then all the agents should receive information about the state of this system; otherwise, there is no learning policy that can guarantee sub-linear regret for all instances of the decentralized MARL problem. Further, when a system is unknown but there is one-directional communication from the agent controlling the unknown system to the other agents, we proposed a MARL algorithm with regret bounded by $\tilde{O}(\sqrt{T})$.  
}

\red{
The MARL algorithm is based on the construction of an auxiliary single-agent LQ problem.
The auxiliary single-agent problem serves as an implicit coordination mechanism among the learning agents. The state of the auxiliary SARL can be interpreted as an estimate of the state of the overall system that each agent computes based on the common information among them. While there is a strong connection between the MARL and auxiliary SARL problems, the MARL problem is not reduced to a SARL problem. In particular, Lemma \ref{lem:expectation_equalities} shows that the costs of the two problems actually differ by a term that depends on the random process $e_t$, which is dynamically controlled by the MARL algorithm. Therefore, the auxiliary SARL problem is not equivalent to the MARL problem. Nevertheless, the proposed MARL algorithm can bound the additional regret due to the process $e_t$ and achieve the same regret order as a SARL algorithm.
}

\red{
The use of the common state estimate plays a key role in the MARL algorithm. The current theoretical analysis uses this common state estimate along with some properties of LQ structure (e.g. certainty equivalence which connects estimates to optimal control \citep{kumar2015stochastic}) to quantify the regret bound. However, certainty equivalence is often used in general systems with continuous state and action spaces as a heuristic with some good empirical performance. This suggests that our algorithm combined with linear approximation of dynamics could potentially be applied to non-LQ systems as a heuristic. That is, each agent constructs an auxiliary SARL with the common estimate as the state, solves this SARL problem heuristically using approximate linear dynamics and/or certainty equivalence, and then modifies the SARL outputs according to the agent's private information.
}

\newpage
\newpage
\bibliography{references}

\newpage
\onecolumn

\appendix

\begin{Large}
\textbf{\begin{center}
Regret Bounds for Decentralized Learning in Cooperative Multi-Agent Dynamical Systems
\\
(Supplementary File)
\end{center}}
\end{Large}

\textbf{Outline}. The supplementary material of this paper is organized as follows.
\begin{itemize}
\item Appendix \ref{app:notation} presents the notation which is used throughout this Supplementary File. 

\item Appendix \ref{app:prelim} presents a set of preliminary results, which are useful in proving the main results of this paper.

\item Appendix \ref{app:proof_thm1} provides the proof of Theorem \ref{thm:negative_results_case1}.
\item Appendix \ref{sec:proof_theorem_2} provides the proof of Theorem \ref{thm:regret_bound_case2}.

\item Appendix \ref{proof_lem:Sigma_convergence} provides the proof of Lemma \ref{lem:Sigma_convergence}. Note that this lemma has been stated in Appendix \ref{sec:proof_theorem_2} and is required for the proof of Theorem \ref{thm:regret_bound_case2}.

\item Appendix \ref{proof_lemma:prob_regret} provides the proof of Lemma \ref{thm:prob_regret}. Note that this lemma has been stated in Appendix \ref{sec:proof_theorem_2} and is required for the proof of Theorem \ref{thm:regret_bound_case2}.

\item Appendix \ref{proof_lem:connection_optimal_cost} provides the proof of 
Lemma \ref{lem:connection_optimal_cost_appendix1}. Note that this lemma, which has been stated in Appendix \ref{sec:proof_theorem_2}, is the rephrased version of  Lemma \ref{lem:connection_optimal_cost} in the main submission. 

\item Appendix \ref{proof_lem:expectation_equalities} provides the proof of 
Lemma \ref{lem:expectation_equalities_appendix1}. Note that this lemma, which has been stated in Appendix \ref{sec:proof_theorem_2}, is the rephrased version of Lemma \ref{lem:expectation_equalities} in the main submission.

\item Appendix \ref{sec:details_SARL_learner} describes the \texttt{SARL} learner $\mathcal{L}$ of some of existing algorithms for the \texttt{SARL} problems in details.

\item Appendix \ref{lemma1_complete} provides two lemmas. The first lemma (Lemma \ref{lm:opt_strategies_appendix}) is the complete version of Lemma \ref{lm:opt_strategies} which describes optimal strategies for the optimal multi-agent LQ problem of Section \ref{sec:optimal_dec_LQ}. The second lemma (Lemma \ref{lm:opt_strategies_centralized_appendix}) describes optimal strategies for the optimal single-agent LQ problem of Section \ref{sec:centralized_LQ}.

\item Appendix \ref{sec:experiments_appendix} provides the details of the experiments in the main submission (Section \ref{sec:experiments_main}).

\item Appendix \ref{proof:lem:regret_bound_extension} provides the proof of Theorem \ref{lem:regret_bound_extension} which extends Theorem \ref{thm:regret_bound_case2} to the case with more than 2 agents.

\item Appendix \ref{sec:two_way} provides the analysis and the results for unknown $\theta_*^1$ and $\theta_*^2$, two-way information sharing ($\gamma^1=\gamma^2=1$).

\end{itemize}
\section{Notation}  
\label{app:notation}
In general, subscripts are used as time indices while superscripts are used to index agents. The collection of matrices $A^1,\ldots, A^n$ (resp. vectors $x^1,\ldots,x^n$) is denoted as $A^{1:n}$ (resp. $x^{1:n}$).
Given column vectors $x^1,\ldots,x^n$, the notation $\vecc(x^{1:n})$ is used to denote the column vector formed by stacking vectors  $x^1,\ldots,x^n$ on top of each other. For two symmetric matrices $A$ and $B$, $A\succeq B$ (resp. $A \succ B$) means that $(A- B)$ is positive semi-definite (PSD) (resp. positive definite (PD)). The trace of matrix $A$ is denoted by $\tr (A)$.

We use $\norm{\cdot}_{\bullet}$ to denote the operator norm of matrices. We use $\norm{\cdot}_2$ to denote the spectral norm, that is, $\norm{M}_2$ is the maximum singular value of a matrix $M$. We use $\norm{\cdot}_1$  and $\norm{\cdot}_\infty$ to denote maximum column sum matrix norm and maximum row sum matrix norm, respectively. More specifically, if $M \in \R^{m \times n}$, then 
$\norm{M}_1 = \max_{1 \leq j \leq n} \sum_{i=1}^m | m_{ij} |$ and $\norm{M}_\infty = \max_{1 \leq i \leq m} \sum_{j=1}^n | m_{ij} |$ where $m_{ij}$ is the entry at the $i$-th row and $j$-th column of $M$. We further use $\norm{\cdot}_{\rm F}$  to denote the Frobenius norm, that is,  $\norm{M}_{\rm F} = \sqrt{\sum_{i=1}^m \sum_{j=1}^n |m_{ij}|^2} = \sqrt{\tr(M^{\tp}M)}$. %and $\norm{\cdot}_*$ to denote the trace norm, that is, $\norm{M}_* = \tr(\sqrt{M^{\tp}M})$. 
The notation $\rho(M)$ refers to the spectral radius of a matrix $M$, i.e., $\rho(M)$ is the largest absolute value of its eigenvalues.

Consider matrices $P,Q,R,A,B$ of appropriate dimensions with $P,Q$ being PSD matrices and $R$ being  a PD matrix. We define $\mathcal{R} (P,Q,R,A,B)$ and $\mathcal{K} (P,R,A,B)$ as follows: 
%
%\vspace{-4mm}
%\begin{small}
\begin{align}
%\label{Omega}
\mathcal{R} (P,Q,R,A,B) := &Q+A^\tp P A- A^\tp P B(R+B^\tp P B)^{-1}B^\tp P A.
\notag
\\
\mathcal{K} (P,R,A,B) 
:= &
-(R+B^\tp P B)^{-1}B^\tp P A.
%\label{Psi}
\notag
\end{align}
%\end{small}
Note that $P = \mathcal{R} (P,Q,R,A,B) $ is the discrete time algebraic Riccati equation.

We use $[P^{\cdot, \cdot}]_{1:4}$ and $\textbf{diag}(P^1,\ldots, P^4)$ to denote the following block matrices,
\begin{align*}
[P^{\cdot, \cdot}]_{1:4} := \begin{bmatrix}
P^{11} & \ldots & \ldots & P^{14}\\ 
\vdots & \ddots & &\vdots \\
\vdots & & \ddots &\vdots \\
P^{41} & \ldots & \ldots & P^{44}
\end{bmatrix}, \quad 
\textbf{diag}(P^1, \ldots, P^4) = 
\begin{bmatrix}
P^{1} & \mathbf{0}  & \ldots & \mathbf{0} \\ 
\mathbf{0}  & \ddots &\ddots &\vdots \\
\vdots & \ddots & \ddots &\mathbf{0}  \\
\mathbf{0}  & \ldots & \mathbf{0}  & P^{4}
\end{bmatrix}
\end{align*}
Further, we use $[P]_{i,i}$ to denote the block matrix located at the $i$-th row partition and $i$-th column partition of $P$. For example, $[\textbf{diag}(P^1,\ldots, P^4)]_{2,2} = P^2$ and $[[P^{\cdot, \cdot}]_{1:4}]_{1,1} = P^{11}$.
\section{Preliminaries}
\label{app:prelim}
First, we state a variant of the Hanson-Wright inequality \citep{hanson1971bound} which can be found in \citet{hsu2012tail}.

\begin{theorem}[\citet{hsu2012tail}] 
\label{thm:hanson_wright}
Let $X \sim \mathcal{N}(0, \mathbf{I})$ be a Gaussian random vector and let $A \in \R^{m \times n}$ and $\Delta:= A^{\tp} A$. For all $z > 0$,
\begin{align}
\prob\big( \norm{AX}_2^2 - \ee [ \norm{AX}_2^2] > 2  \norm{\Delta}_{\rm F} \sqrt{z} + 2 \norm{\Delta}_2 z  \big) \leq \exp(-z).
\end{align}
\end{theorem}

\begin{lemma}
\label{lem:matrix_inequality}
Let $A \in \R^{l \times m}, B \in \R^{m \times n}$. Then, $\norm{AB}_{\rm F} \leq \norm{A}_2 \norm{B}_{\rm F} $.
\end{lemma}

\begin{proof}
Let $B=[b_1,\ldots,b_n]$ be the column partitioning of $B$. Then,
\begin{align}
\norm{AB}_{\rm F}^2=\sum_{i=1}^n \norm{Ab_i}_2^2 \leq  \norm{A}_2^2 \sum_{i=1}^n \norm{b_i}_2^2= \norm{A}_2^2 \norm{B}_{\rm F}^2,
\end{align}
where the first equality follows from the definition of Frobenius norm, the first inequality is correct because the operator norm is a sub-multiplicative matrix norm, and the last equality follows from the definition of Frobenius norm.
\end{proof}

Using Theorem \ref{thm:hanson_wright} and Lemma \ref{lem:matrix_inequality}, we can state the following result.

\begin{lemma}
\label{lem:modified_hanson_wright}
Let $X \sim \mathcal{N}(0, \mathbf{I})$ be a Gaussian random vector and let $A \in \R^{m \times n}$. 
Then for any $\delta \in (0, 1/e)$, we have
\begin{align}
\norm{AX}_2^2 - \tr(A^{\tp} A) \leq  4 \norm{A}_2 \norm{A}_{\rm F} \log(\frac{1}{\delta}),
\end{align}
with probability at least $1 - \delta$.
\end{lemma}

\begin{proof}
Since $X \sim \mathcal{N}(0, \mathbf{I})$, from Theorem \ref{thm:hanson_wright}, for any $z>1$, we have with probability at least $1 - \exp(-z)$,
\begin{align}
\norm{AX}_2^2 - \tr(A^{\tp} A) & \leq  2  \norm{\Delta}_{\rm F} \sqrt{z} + 2 \norm{\Delta}_2 z
\leq 2  \norm{A}_2 \norm{A}_{\rm F} \sqrt{z} + 2 \norm{A}_2 \norm{A}_{2} z \notag \\
& \leq 
2 \norm{A}_2 \norm{A}_{\rm F} z + 2 \norm{A}_2 \norm{A}_{\rm F} z
 \leq 
4 \norm{A}_2 \norm{A}_{\rm F} z,
\end{align}
where the second inequality is correct because of Lemma \ref{lem:matrix_inequality} and the third inequality is correct because $z>1$ and $\norm{A}_{2} \leq \norm{A}_{\rm F}$. Now by choosing $z= \log(\frac{1}{\delta})$ where $\delta \in (0, 1/e)$ the correctness of Lemma \ref{lem:modified_hanson_wright} is obtained.
\end{proof}

\begin{lemma}[Lemma 5.6.10 \citep{matrix_analysis_horn}]
Let $A \in \R^{n \times n}$ and $\epsilon >0$ be given. There is a matrix norm $\norm{\cdot}_{\bullet}$ such that 
$\rho(A) \leq \norm{A}_{\bullet} \leq \rho(A) + \epsilon$.
\end{lemma}
The above lemma implies the following results.
\begin{corollary}
\label{cor:norm_for_rho}
Let $A \in \R^{n \times n}$ and $\rho(A)  <1$. Then, there exists some matrix norm $\norm{\cdot}_{\bullet}$ such that 
$\norm{A}_{\bullet} < 1$.
\end{corollary}

\begin{lemma}

Let $A$ be a $d \times d$ block matrix where $A_{i,j} \in \R^{n \times n}$ denotes the block matrix at the $i$-th row partition and $j$-th column partition. Then, 
\begin{align}
\norm{A}_\infty  &= \max_{j=1,\ldots, d} \norm{\sum_{i=1}^{d} \tilde A_{i,j} }_{\infty}, \quad
\norm{A}_1 = \max_{i=1,\ldots, d} \norm{\sum_{j=1}^{d} \tilde A_{i,j} }_{1},
\end{align}
where matrix $ \tilde A_{i,j}$ is the entry-wise absolute value of matrix $A_{i,j}$.
\end{lemma}
\begin{proof}
We prove the equality for $\norm{A}_\infty$. The proof for $\norm{A}_1$ can be obtained in a similar way. Note that,
\begin{align}
\max_{i=1,\ldots, d} \norm{\sum_{j=1}^{d} \tilde A_{i,j} }_{\infty}
= \max_{i=1,\ldots, d}  \max_{k_i=1,\ldots, n} \sum_{j=1}^{d} \sum_{k_j=1}^{n}  | \tilde a_{k_ik_j} |  = \max_{1 \leq i \leq nd} \sum_{j=1}^{nd} | a_{ij} | = \norm{A}_\infty
\end{align}
where $\tilde a_{ij}$ is the entry at the $i$-th row and $j$-th column of $\tilde A$. 
\end{proof}

\section{Proof of Theorem \ref{thm:negative_results_case1}}
\label{app:proof_thm1}
We want to show that there is no algorithm that can achieve a lower-bound better than $\Omega(T)$ on the regret of all instances of the \texttt{MARL1} problem. Equivalently, we can show that for any algorithm, there is an instance of the \texttt{MARL1} problem whose regret is at least $\Omega(T)$. To this end, consider an instance of the \texttt{MARL1} problem where the systems dynamics and the cost function are described as follows\footnote{Note that for simplicity, we have assumed here that there is no noise in the both systems.},
\begin{align}
\label{example:dynamics}
&x_{t+1}^1 = u_t^1, \quad x_{t+1}^2 = a_*^2 x_t^2, \quad \quad x_0^1 = x_0^2 = 1, \\
&c(x_t,u_t) = (x_t^1 - x_t^2)^2 + (u_t^1 - 0.5 u_t^2)^2.
\label{example:cost}
\end{align}
We assume that the only unknown parameter is $a_*^2$. Note that for any $a_*^2 \in (-1,1)$, the above problem satisfies Assumption \ref{assum:det_stb}. By using \eqref{example:dynamics}, the cost function of \eqref{example:cost} can be rewritten as,
\begin{align}
&c(x_t,u_t) = (u_{t-1}^1 - (a_*^2)^t)^2 + (u_t^1 - 0.5 u_t^2)^2.
\label{example:cost2}
\end{align} 
If $a_*^2$ is known to the both controllers, one can easily show that the optimal infinite horizon average cost is 0 and it is achieved by setting $u_t^1 = (a_*^2)^{t+1}$ and $u_t^2 = 2 (a_*^2)^{t+1}$.

If $a_*^2$ in unknown, the regret of any policy $\pi$ can be written as\footnote{Note that at each time $t$, what agent 1 observes  is its previous action (that is, $x_t^1 = u_{t-1}^1$) and what agent 2 observes is a fixed number (that is, $x_t^2 = (a_*^2)^{t}$). In other words, agents do not get any new feedback about their respective systems. Therefore, any policy $\pi$ is indeed an open-loop policy.},
\begin{align}
R(T, \pi) &= \sum_{t=0}^{T-1} c(x_t,u_t) =  (u_0^1 - 0.5 u_0^2)^2  + \sum_{t=1}^{T-1} \left[ (u_{t-1}^1 - (a_*^2)^t)^2 + (u_t^1 - 0.5 u_t^2)^2  \right]  \notag \\ & \geq \sum_{t=1}^{T-1} (u_{t-1}^1 - (a_*^2)^t)^2,
\label{eq:negative_example}
\end{align}
where the first equality is correct due to the fact that the optimal infinite horizon average cost is 0, the second equality is correct because of \eqref{example:cost2} and the fact that $x_0^1 = x_0^2 = 1$, and the first inequality is correct because $(u_t^1 - 0.5 u_t^2)^2 \geq 0$. Now, we show that for any policy $\pi$, there is a value for $a_*^2$ such that $R(T, \pi)  \geq \Omega(T)$. This is equivalent to show that $\sup_{a_*^2 \in (-1,1)} R(T,\pi) \geq \Omega(T)$. This can be shown as follows,
\begin{align}
& \sup_{a_*^2 \in (-1,1)} R(T,\pi) \geq \sup_{a_*^2 \in (-1,1)} \sum_{t=1}^{T-1} (u_{t-1}^1 - (a_*^2)^t)^2
\geq \frac{1}{2} \sum_{t=1}^{T-1} (u_{t-1}^1 - 0)^2 +  \frac{1}{2} \sum_{t=1}^{T-1} (u_{t-1}^1 - \alpha^t)^2 \notag \\
& = \sum_{t=1}^{T-1} \left[ (u_{t-1}^1 - \frac{\alpha^t}{2})^2 + \frac{\alpha^{2t}}{2} - \frac{\alpha^{2t}}{4} \right] \geq \sum_{t=1}^{T-1} \frac{\alpha^{2t}}{4} = \frac{\alpha^2(1 - \alpha^{2T})}{4(1 -\alpha^{2})},\quad \forall \alpha \in (-1, 1), 
\label{eq:neg_example_sum_regret}
\end{align}
where the first inequality is correct because supremum over a set is greater than or equal to expectation with respect to any distribution over that set. Further, the second equality is correct because $(u_{t-1}^1 - \frac{1}{2})^2 \geq 0$. Since \eqref{eq:neg_example_sum_regret} is true for any $\alpha \in (-1, 1)$, it holds also for limit when $\alpha \rightarrow 1^{-}$. By taking the limit, we can obtain
\begin{align}
& \sup_{a_*^2 \in (-1,1)} R(T,\pi) \geq \lim_{\alpha \to 1^{-}} \frac{\alpha^2(1 - \alpha^{2T})}{4(1 -\alpha^{2})} = \frac{T}{4} = \Omega(T).
\label{eq:neg_example_sum_regret}
\end{align}

This completes the proof.

\section{Proof of Theorem \ref{thm:regret_bound_case2}}
\label{sec:proof_theorem_2}
We first state some preliminary results in the following lemmas which will be used in the proof of Theorem \ref{thm:regret_bound_case2}. 

\begin{lemma}
\label{lem:Sigma_convergence}
Let $s_t$ be a random process that evolves as follows,
\begin{align}
s_{t+1} = C s_t + v_t, \quad s_0 = 0,
\label{eq:s_t_process} 
\end{align}
where $v_t$, $t \geq 0$, are independent Gaussian random vectors with zero-mean and covariance matrix $\cov(v_t) = \mathbf{I}$. Further, let $C=A_*^{2} + B_*^{2} \tilde K^2(\theta_*^2)$ and define $\Sigma_t = \cov(s_t)$, then the sequence of matrices $\Sigma_t$, $t \geq0$, is increasing\footnote{Note that increasing is in the sense of partial order $\succeq$, that is, $\Sigma_0 \preceq \Sigma_1 \preceq \Sigma_2 \preceq \ldots$} and it converges to a PSD matrix $\Sigma$ as $t \to \infty$. Further, $C$ is a stable matrix, that is, $\rho(C) < 1$.
\end{lemma}
\begin{proof}
See Appendix \ref{proof_lem:Sigma_convergence} for a proof.
\end{proof}

\begin{lemma}
\label{thm:prob_regret}
Let $s_t$ be a random process defined as in Lemma \ref{lem:Sigma_convergence}. Let $D$ be a positive semi-definite (PSD) matrix. Then for any $\delta \in (0, 1/e)$, with probability at least $1 - \delta$, 
\begin{align}
\sum_{t=1}^T [s_t^{\tp} D s_t - \tr(D \Sigma)] \leq \log(\frac{1}{\delta}) \tilde K \sqrt{T}. 
\end{align}
\end{lemma}

\begin{proof}
See Appendix \ref{proof_lemma:prob_regret} for a proof.
\end{proof}

We now proceed in two steps:
\begin{itemize}
\item Step 1: Showing the connection between the auxiliary \texttt{SARL} problem and the \texttt{MARL2} problem
\item Step 2: Using the \texttt{SARL} problem to bound the regret of the \texttt{MARL2} problem
\end{itemize}

\subsection*{Step 1: Showing the connection between the auxiliary \texttt{SARL} problem and the \texttt{MARL2} problem}

First, we present the following lemma that connects the optimal infinite horizon average cost $J^{\diamond}(\theta_*^{1,2})$ of the auxiliary \texttt{SARL} problem when $\theta_*^{1,2}$ are known (that is, the auxiliary single-agent LQ problem of Section \ref{sec:centralized_LQ}) and the optimal infinite horizon average cost $J(\theta_*^{1,2})$ of the \texttt{MARL2} problem when $\theta_*^{1,2}$ are known  (that is, the multi-agent LQ problem of Section \ref{sec:optimal_dec_LQ}).
\begin{lemma}[rephrased version of Lemma \ref{lem:connection_optimal_cost}]
\label{lem:connection_optimal_cost_appendix1}
Let $J^{\diamond}(\theta_*^{1,2})$ be the optimal infinite horizon average cost of the auxiliary \texttt{SARL} problem, $J(\theta_*^{1,2})$ be the optimal infinite horizon average cost of the \texttt{MARL2} problem, and $\Sigma$ be as defined in Lemma \ref{lem:Sigma_convergence}. Then,
\begin{align}
J(\theta_*^{1,2}) = J^{\diamond}(\theta_*^{1,2}) + \tr(D \Sigma),
\label{eq:optimal_costs_relation_appendix1}
\end{align}
where we have defined $D:= Q^{22} + (\tilde K^2(\theta_*^2))^{\tp} R^{22} \tilde K^2(\theta_*^2)$.
\end{lemma}

\begin{proof}
See Appendix \ref{proof_lem:connection_optimal_cost} for a proof.
\end{proof}

Next, we provide the following lemma that shows the connection between the cost $c(x_t,u_t)$ in the \texttt{MARL2} problem under the policy of the \texttt{AL-MARL} algorithm and the cost $
c(x_t^{\diamond}, u_t^{\diamond})$ in the auxiliary \texttt{SARL} problem under the policy of the \texttt{AL-SARL} algorithm.
\begin{lemma}[rephrased version of Lemma \ref{lem:expectation_equalities}]
\label{lem:expectation_equalities_appendix1}
At each time $t$, the following equality holds between the cost under the policies of the \texttt{AL-SARL} and the \texttt{AL-MARL} algorithms,
\begin{align}
c(x_t,u_t) \vert_{\texttt{AL-MARL}} &= c(x_t^{\diamond}, u_t^{\diamond})\vert_{\texttt{AL-SARL}} + e_t^{\tp} D e_t,
\label{eq:expected_cost_cent_dec_appendix1}
\end{align}
where $e_t = x_t^2 - \check x_t^2$ and $D= Q^{22} + (\tilde K^2(\theta_*^2))^{\tp} R^{22} \tilde K^2(\theta_*^2)$.
\end{lemma}

\begin{proof}
See Appendix \ref{proof_lem:expectation_equalities} for a proof.
\end{proof}
\subsection*{Step 2: Using the \texttt{SARL} problem to bound the regret of the \texttt{MARL2} problem}
In this step, we use the connection between the auxiliary \texttt{SARL} problem  and our \texttt{MARL2} problem, which was established in Step 1, to prove Theorem \ref{thm:regret_bound_case2}. Note that from the definition of the regret in the \texttt{MARL} problem given by \eqref{eq:regret_decentralized}, we have,
\begin{align}
&R(T, \texttt{AL-MARL}) = \sum_{t=0}^{T-1} \left[ c(x_t,u_t) \vert_{\texttt{AL-MARL}} - J(\theta_*^{1,2})   \right] \notag \\
%& =  \sum_{t=0}^{T-1}  \left[c(x_t^{\diamond}, u_t^{\diamond}) \vert_{\texttt{AL-SARL}} + e_t^{\tp} D e_t \right]  -    \sum_{t=0}^{T-1} \left[ J^{\diamond}(\theta_*^{1,2}) + \tr(D \Sigma) \right] \notag \\
& =  \sum_{t=0}^{T-1}  \left[c(x_t^{\diamond}, u_t^{\diamond}) \vert_{\texttt{AL-SARL}} - J^{\diamond}(\theta_*^{1,2}) \right]  +  \sum_{t=0}^{T-1} \left[ e_t^{\tp} D e_t - \tr(D \Sigma) \right]  \leq R^{\diamond}(T, \texttt{AL-SARL}) +\log(\frac{1}{\delta}) \tilde K \sqrt{T}
 \label{eq:regret_inequality_case2_appendix}
\end{align} where the second equality is correct because of  Lemma \ref{lem:connection_optimal_cost_appendix1} and Lemma \ref{lem:expectation_equalities_appendix1}. Further, if we define $v_t:= w_t^2$, $e_t$ has the same dynamics as $s_t$ in Lemma \ref{thm:prob_regret}. Then, the last inequality is correct because of Lemma \ref{thm:prob_regret} and the definition of the regret in the \texttt{SARL} problem given by in \eqref{eq:regret_centralized}. This proves the statement of Theorem \ref{thm:regret_bound_case2}.

\section{Proof of Lemma \ref{lem:Sigma_convergence}}
\label{proof_lem:Sigma_convergence}
First, note that $\Sigma_t$ can be sequentially calculated as $\Sigma_{t+1} = \mathbf{I} + C \Sigma_t C^{\tp}$ with $\Sigma_0 = \mathbf{0}$. 
Now, we use induction to show that the sequence of matrices $\Sigma_t$, $t \geq 0$, is increasing. First, we can write $\Sigma_{t+1} - \Sigma_t = C (\Sigma_{t} - \Sigma_{t-1}) C^{\tp}$. Then, since $\Sigma_0 = \mathbf{0}$ and $\Sigma_1 = \mathbf{I} \succeq \mathbf{0}$, we have $\Sigma_{1} - \Sigma_0 \succeq \mathbf{0}$. Now, assume that $\Sigma_{t} - \Sigma_{t-1} \succeq \mathbf{0}$. Then, it is easy to see that $\Sigma_{t+1} - \Sigma_t = C (\Sigma_{t} - \Sigma_{t-1}) C^{\tp} \succeq \mathbf{0}$.

To show that the sequence of matrices $\Sigma_t$, $t \geq 0$, converges to $\Sigma$ as $t \to \infty$, first we show that $C$ is stable, that is, $\rho(C)<1$. Note that $C=A_*^{2} + B_*^{2} \tilde K^2(\theta_*^2)$ where from \eqref{eq:tildeK_infinite_appendix}, we have
\begin{align}
\tilde K^2(\theta_*^2) &= \mathcal{K} (\tilde P^2 (\theta_*^2) ,R^{22},A_*^{2},B_*^{2}), \notag \\
\tilde P^2(\theta_*^2) &= \mathcal{R} (\tilde P^2(\theta_*^2) , Q^{22},R^{22},A_*^{2},B_*^{2}).
\label{eq:tildeP_proof}  
\end{align}
Then, from Assumption \ref{assum:det_stb}, $(A_*, B_*)$ is stabilizable and since both of $A_*$ and $B_*$ are block diagonal matrices, $(A_*^{2}, B_*^{2})$ is stabilizable. Hence, we know from \citet[Theorem 2.21]{costa2006discrete} that $\rho(C)<1$. Since $C$ is stable, the converges of the sequence of matrices $\Sigma_t$, $t \geq0$, can be concluded from \citet[Chapter 3.3]{kumar2015stochastic}.

\section{Proof of Lemma \ref{thm:prob_regret}}
\label{proof_lemma:prob_regret}
In this proof, we use superscripts to denote exponents.
\subsection*{Step 1:}
\begin{lemma}
\label{lem:s_t_matrix_form}
Let $s_t$ be as defined in \eqref{eq:s_t_process}. Then,
\begin{align}
\sum_{t=1}^T [s_t^{\tp} D s_t - \tr(D \Sigma_t)] = \bar v^{\tp} L^{\tp} L \bar v - \tr(L^{\tp}L ),
\label{eq:compact_form}
\end{align}
where $L = \bar D^{1/2} \bar C$ and 
\begin{align}
\bar v = \begin{bmatrix}
v_0 \\ v_1 \\ \vdots \\ v_{T-1}
\end{bmatrix}, \quad
\bar C = \begin{bmatrix}
\mathbf{I} & \mathbf{0} & \ldots & \ldots & \mathbf{0} \\
C & \mathbf{I} &\mathbf{0} &  & \vdots \\
C^2 & C & \mathbf{I} & \ddots & \vdots \\
\vdots & \vdots & \ddots & \ddots & \mathbf{0} \\
C^{T-1} & C^{T-2} & \ldots & C & \mathbf{I}
\end{bmatrix}, \quad
\bar D = \begin{bmatrix}
D & \mathbf{0} & \ldots & \ldots & \mathbf{0} \\
\mathbf{0}  & D &\mathbf{0} &  & \vdots \\
\mathbf{0}  & \mathbf{0}  & D & \ddots & \vdots \\
\vdots & \vdots & \ddots & \ddots & \mathbf{0} \\
\mathbf{0}  & \mathbf{0}  & \ldots & \mathbf{0}  & D
\end{bmatrix}.
\label{eq:bar_matrices}
\end{align}
\end{lemma}
\begin{proof}
First note that from \eqref{eq:s_t_process}, $s_t$ can be written as,
\begin{align}
s_t = \sum_{i=0}^{t-1} C^{t-1-i} v_i + \sum_{i=t}^{T-1} \mathbf{0} \times v_i 
= \begin{bmatrix}
C^{t-1} & C^{t-2} & \ldots & C & \mathbf{I} & \mathbf{0} & \ldots & \mathbf{0} 
\end{bmatrix} \bar v.
\end{align}
Furthermore, since $D$ and consequently, $\bar D$ are PSD matrices, there exists $\bar D^{1/2}$ such that $\bar D = (\bar D^{1/2})^{\tp} \bar D^{1/2}$ (similarly, $D = (D^{1/2})^{\tp} D^{1/2}$). Then, the correctness of \eqref{eq:compact_form} is obtained through straightforward algebraic manipulations. 
\end{proof}

\subsection*{Step 2:}
Since from Lemma \ref{lem:s_t_matrix_form}, $\bar v$ in is Gaussian, we can apply Lemma \ref{lem:modified_hanson_wright} to bound $\bar v^{\tp} L^{\tp} L \bar v -  \tr(L^{\tp}L )$ as follows. For any $\delta \in (0, 1/e)$,  we have with probability at least $1 - \delta$, 
\begin{align}
\bar v^{\tp} L^{\tp} L \bar v -  \tr(L^{\tp}L ) \leq  4 \norm{L}_2 \norm{L}_{\rm F} \log(\frac{1}{\delta}).
\end{align}

\subsection*{Step 3:}
In this step, we find an upper-bound for $\norm{L}_{\rm F}$. To this end, first note that by definition, we have 
$\norm{L}_{\rm F} = \sqrt{\tr(L^{\tp}L)}$. From \eqref{eq:bar_matrices}, we can write $L$ as follows,
\begin{align}
L = \begin{bmatrix}
D^{1/2} & \mathbf{0} & \ldots & \ldots & \mathbf{0} \\
D^{1/2}C & D^{1/2} &\mathbf{0} &  & \vdots \\
D^{1/2}C^2 & D^{1/2}C & D^{1/2} & \ddots & \vdots \\
\vdots & \vdots & \ddots & \ddots & \mathbf{0} \\
D^{1/2}C^{T-1} & D^{1/2}C^{T-2} & \ldots & D^{1/2} C & D^{1/2}
\end{bmatrix}.
\label{eq:L_matrix}
\end{align}
Then, using \eqref{eq:L_matrix}, we have
\begin{align}
L^{\tp} L = \begin{bmatrix}
\sum_{i=0}^{T-1}  (C^{i})^{\tp}D C^{i}   & \mathbf{\times} & \ldots & \ldots & \mathbf{\times} \\
\mathbf{\times}  & \sum_{i=0}^{T-2}  (C^{i})^{\tp}D C^{i} &\mathbf{\times} &  & \vdots \\
\mathbf{\times}  & \mathbf{\times}  & \sum_{i=0}^{T-3}  (C^{i})^{\tp}D C^{i} & \ddots & \vdots \\
\vdots & \vdots & \ddots & \ddots & \mathbf{\times} \\
\mathbf{\times}  & \mathbf{\times}  & \ldots & \mathbf{\times}  & D
\end{bmatrix}.
\label{eq:L_L_tp}
\end{align}
Now, from \eqref{eq:L_L_tp} and the fact that trace is a linear operator, we can write,
\begin{align}
\tr(L^{\tp}L) = \sum_{j=0}^{T-1} \sum_{i=0}^j \tr \big (  (C^{i})^{\tp}D C^{i} \big ).
\label{eq:trace_LL}
\end{align}
In the following, we find an upper-bound for $\tr \big (  (C^{i})^{\tp}D C^{i} \big )$. Since $D$ is a PSD matrix, we can write it as $D = \sum_{l=1}^r d_l d_l^T$ where $r$ is rank of matrix $D$. By using this, we can have,
\begin{align}
\tr \big (  (C^{i})^{\tp}D C^{i} \big ) &= \tr \big (  (C^{i})^{\tp} \sum_{l=1}^r d_l d_l^T  C^{i} \big )
\stackrel{(*1)}{=} \sum_{l=1}^r  \tr \big (  (C^{i})^{\tp}  d_l d_l^T  C^{i} \big )
\stackrel{(*2)}{=} \sum_{l=1}^r  \norm{ (C^{i})^{\tp}  d_l}_2^2 \notag \\
& \leq  \sum_{l=1}^r  \norm{ (C^{i})^{\tp}}_2^2 \norm{d_l}_2^2
\leq  \norm{C}_2^{2i}  \sum_{l=1}^r  \norm{d_l}_2^2
=  \norm{C}_2^{2i} \tr(D)
\stackrel{(*3)}{\leq } \beta \norm{C}_{\bullet}^{2i} \tr(D) \notag \\
& \stackrel{(*4)}{=} \beta \alpha^{2i}  \tr(D),
\label{eq:trace_CC}
\end{align}
where $(*1)$ is correct because $\tr(\cdot)$ is a linear operator and $(*2)$ is correct because if $v$ is a column vector, then $\tr(v v^{\tp}) = \norm{v}_2^2$. Further, $(*3)$ is correct because any two norms on a finite dimensional space are equivalent. In other words, for any norm $\norm{\cdot}_{\bullet}$, there is a number $\beta$ such that $\norm{C}_2 \leq \beta \norm{C}_{\bullet}$. Note that $\beta$ is independent of $T$. Now, let the second norm (that is, $\norm{\cdot}_{\bullet}$) be the norm for which $\norm{C}_{\bullet} <1$ (note that since $\rho(C) <1$, the existence of this norm follows form Corollary \ref{cor:norm_for_rho}). Then, the correctness of $(*4)$ is resulted by defining $\alpha:= \norm{C}_{\bullet} <1$.

From \eqref{eq:trace_CC}, we can write,
\begin{align}
\sum_{i=0}^j \tr \big (  (C^{i})^{\tp}D C^{i} \big ) \leq \beta \tr(D) \sum_{i=0}^j \alpha^{2i}  \leq 
\beta \tr(D) \sum_{i=0}^{\infty} \alpha^{2i} = \beta \frac{\tr(D)}{1-\alpha^2}.
\label{eq:sum_trace_CC}
\end{align}
Finally, using \eqref{eq:trace_LL} and \eqref{eq:sum_trace_CC}, we have $\tr(L^{\tp}L)  \leq T \beta \frac{\tr(D)}{1-\alpha^2}$ which means that 
\begin{align}
\norm{L}_{\rm F} = \sqrt{\tr(L^{\tp}L)} \leq \sqrt{T  \beta \frac{\tr(D)}{1-\alpha^2}}.
\end{align}

\subsection*{Step 4:}
In this step, we find an upper-bound for $\norm{L}_2$. We use $\norm{L}_2 \leq \sqrt{\norm{L}_1 \norm{L}_\infty}$ to bound $\norm{L}_2$ \citep{golub1996matrix}. To this end, we calculate $\norm{L}_1$ and  $\norm{L}_\infty$. 

\underline{Scalar case:}

Because of the special structure of matrix $L$ in \eqref{eq:L_matrix}, these two matrix norms are the same and they are equal to sum of the entries of the first column,
\begin{align}
\norm{L}_1  = \norm{L}_\infty = \sum_{i=0}^{T-1} | D^{1/2}C^{i} | \leq D^{1/2} \sum_{i=0}^{T-1} |C^{i}|
\leq D^{1/2} \sum_{i=0}^{T-1} |C|^i 
\leq D^{1/2} \sum_{i=0}^{\infty} \alpha^i = \frac{D^{1/2}}{1- \alpha}.
\label{eq:norm_1_infty}
\end{align}
Using \eqref{eq:norm_1_infty}, we can bound $\norm{L}_2$ as follows,
\begin{align}
\norm{L}_2 \leq \frac{D^{1/2}}{1- \alpha}.
\label{eq:norm_2_inequality}
\end{align}

\underline{Matrix case:}

Since $L$ is a $T \times T$ block matrix, we can write
\begin{align}
\norm{L}_\infty  &= \max_{i=1,\ldots, T} \norm{\sum_{j=1}^{T} \tilde L_{i,j} }_{\infty} 
\stackrel{(*1)}{=} \norm{\sum_{j=1}^{T} \tilde L_{i,1} }_{\infty}  
\stackrel{(*2)}{\leq}
\sum_{j=1}^{T1}  \norm{ \tilde L_{i,1}}_{\infty} 
 \stackrel{(*3)}{\leq} \sqrt{K} \sum_{j=1}^{T} \norm{\tilde L_{i,1} }_{2} \notag \\
&  \stackrel{(*4)}{\leq} \sqrt{K} \sum_{j=1}^{T}  \norm{ \tilde L_{i,1}}_{\rm F}
 \stackrel{(*5)}{=} \sqrt{K} \sum_{j=1}^{T}  \norm{L_{i,1}}_{\rm F}
  \stackrel{(*6)}{=} \sqrt{K} \sum_{i=0}^{T-1}  \norm{ D^{1/2}C^{i} }_{\rm F} \notag \\
  & \stackrel{(*7)}{=} \sqrt{K} \sum_{i=0}^{T-1}  \sqrt{\tr \big (  (C^{i})^{\tp}D C^{i} \big ) }  \notag \\
& \stackrel{(*8)}{\leq} \sqrt{K} \sum_{i=0}^{T-1}  \alpha^{i} \sqrt{\beta \tr (D)}
 \leq \sqrt{K \beta \tr(D)} \sum_{i=0}^{\infty}  \alpha^{i} = \frac{\sqrt{K \beta \tr(D)}}{1- \alpha},
\label{eq:norm_1_infty_matrix}
\end{align}
where matrix $ \tilde L_{i,j}$ is the entry-wise absolute value of matrix $L_{i,j}$. Note that $(*1)$ is correct because the maximum is achieved by setting $j=1$ (i.e, the first column partition) and $(*2)$ is correct because of the sub-additive property of the norm. For $(*3)$, first note that for any matrix $M \in \R^{n \times n}$, we have $\norm{M}_\infty \leq \sqrt{n} \norm{M}_{2}$. If we define $K$ to be maximum of size of matrices $\tilde L_{i,j}$, then the correctness of $(*3)$ is resulted. Further, $(*4)$ is correct because for any matrix $M$, $\norm{M}_2 \leq \norm{M}_{\rm F}$, $(*5)$ is correct because the Frobenius norm of a matrix and its entry-wise absolute value is the same, $(*6)$ is correct because $L_{i,1} = D^{1/2}C^{i}$, and $(*7)$ follows from the definition of the Frobenius norm. Finally, $(*8)$ is correct because of \eqref{eq:trace_CC}. Similarly, we can show that $\norm{L}_1 \leq  \frac{\sqrt{K \beta \tr(D)}}{1- \alpha}$. Hence, we can bound $\norm{L}_2$ as follows,
\begin{align}
\norm{L}_2 \leq \frac{\sqrt{K \beta \tr(D)}}{1- \alpha}.
\label{eq:norm_2_inequality_matrix}
\end{align}

\subsection*{Step 5:}
By combining the results of Steps 1 to 4, we have with probability at least $1 - \delta$, 
\begin{align}
\sum_{t=1}^T [s_t^{\tp} D s_t - \tr(D \Sigma)]  &= 
\sum_{t=1}^T [s_t^{\tp} D s_t - \tr(D \Sigma_t)] + \sum_{t=1}^T [ \tr(D \Sigma_t) - \tr(D \Sigma)] 
\notag \\
& \leq 
\sum_{t=1}^T [s_t^{\tp} D s_t - \tr(D \Sigma_t)] \leq 
4 \frac{\sqrt{K \beta \tr(D)}}{1- \alpha}  \sqrt{T \beta \frac{\tr(D)}{1-\alpha^2}} \log(\frac{1}{\delta}),
\end{align}
where the first inequality is correct because from Lemma \ref{lem:Sigma_convergence}, the sequence of matrices $\Sigma_t$ is increasing, that is, $\Sigma - \Sigma_t \succeq \mathbf{0}$ and $D$ is positive semi-definite, and consequently, $\tr (D (\Sigma_t - \Sigma) )  \leq 0$. Define $\tilde K := \frac{4 \beta \sqrt{K} \tr(D)}{(1-\alpha) \sqrt{1-\alpha^2}}$, then the correctness of Lemma \ref{thm:prob_regret} is obtained. 

\section{Proof of Lemma \ref{lem:connection_optimal_cost_appendix1} (Lemma \ref{lem:connection_optimal_cost})}
\label{proof_lem:connection_optimal_cost}
Let $\pi^{\diamond*}$ be optimal policy for the auxiliary \texttt{SARL} problem when $\theta_*^{1,2}$ are known. Then, the optimal infinite horizon average cost under $\theta_*^{1,2}$ of this auxiliary \texttt{SARL} problem can be written as,
\begin{align}
J^{\diamond}(\theta_*^{1,2})  = \limsup_{T\rightarrow\infty} \frac{1}{T} \sum_{t=0}^{T-1} \ee^{\pi^{\diamond*}}[c(x_t^{\diamond}, u_t^{\diamond}) | \theta_*^{1,2}].
\label{eq:centralized_cost_def}
\end{align}
Under the optimal policy $\pi^{\diamond*}$ (see Lemma \ref{lm:opt_strategies_centralized_appendix} in Appendix \ref{lemma1_complete}), $u_t^{\diamond} = K(\theta_*^{1,2}) x_t^{\diamond}$ and hence, the dynamics of $x_t^{\diamond}$ in \eqref{Model:system_LQ} can be written as,
\begin{align}
x_{t+1}^{\diamond} = \Big(A_* + B_* K(\theta_*^{1,2}) \Big) x_t^{\diamond} + \vecc(w_t^1, \mathbf{0}).
\label{eq:X_diamond_dynamics}
\end{align}
Further, let $\pi^*$ be optimal policy for the \texttt{MARL} problem when $\theta_*^{1,2}$ is known. Then, from \eqref{eq:optimal_cost_decentralized} we have,
\begin{align}
J(\theta_*^{1,2})  = \limsup_{T\rightarrow\infty} \frac{1}{T} \sum_{t=0}^{T-1} \ee^{\pi^*}[c(x_t, u_t) | \theta_*^{1,2}].
\label{eq:optimal_cost_dec}
\end{align}
From Lemma \ref{lm:opt_strategies}, we know that under the optimal policy $\pi^{*}$, 
\begin{align}
u_t = K(\theta_*^{1,2}) \bar x_t + \vecc(\mathbf{0}, \tilde K^2 (\theta_*^2)e_t),
\label{eq:optimal_U_dec}
\end{align}
where we have defined $\bar x_t := \vecc(x_t^1, \hat x_t^2)$, $e_t := x_t^2 - \hat x_t^2$, and we have $K(\theta_*^{1,2}) = \begin{bmatrix}
K^1(\theta_*^{1,2}) \\ K^2(\theta_*^{1,2})
\end{bmatrix}$ from Lemma \ref{lm:opt_strategies_appendix} in the Appendix \ref{lemma1_complete}. 
Then, from the dynamics of $x_t$ in \eqref{Model:system_overall} and update equation for $\hat x_t^2$ in \eqref{eq:estimator_t_infinite}, we can write 
\begin{align}
x_{t+1} =  \bar x_{t+1} +  \vecc(\mathbf{0}, e_{t+1}),
\label{eq:x_tplus1_transformed}
\end{align}
where
\begin{align}
\label{eq:X_bar_dynamics_dec}
\bar x_{t+1} &= \Big( A_* + B_* K(\theta_*^{1,2}) \Big) \bar x_t + \vecc(w_t^1, \mathbf{0}), \quad e_{t+1} = C e_{t} + w_t^2.
\end{align}
Note that we have defined $C = A_*^{2} + B_*^{2} \tilde K^2 (\theta_*^2)$. Now by comparing \eqref{eq:X_diamond_dynamics} and \eqref{eq:X_bar_dynamics_dec} and the fact that both $x_1^{\diamond}$ and $\bar x_1$ are equal, we can see that for any time $t$,
\begin{align}
\bar x_{t+1} = x_{t+1}^{\diamond}.
\label{eq:equality_Xs}
\end{align}
Now, we can use the above results to write $\ee^{\pi^*}[c(x_t, u_t) | \theta_*^{1,2}]$ as follows,
\begin{align}
&\ee^{\pi^*}[c(x_t, u_t) | \theta_*^{1,2}] = \ee^{\pi^*}[ x_t^\tp Q x_t + (u_t)^{\tp} R u_t | \theta_*^{1,2}] \notag \\
&= \ee^{\pi^*}[ \bar x_t^\tp Q \bar x_t + (K(\theta_*^{1,2}) \bar x_t)^{\tp} R K(\theta_*^{1,2}) \bar x_t \vert \theta_*^{1,2}]  + \ee[e_t^\tp D  e_t| \theta_*^{1,2}] \notag \\
&=  \ee^{\pi^{\diamond*}}[ (x_t^{\diamond})^\tp Q x_t^{\diamond} + (K(\theta_*^{1,2}) x_t^{\diamond})^{\tp} R K(\theta_*^{1,2}) x_t^{\diamond} \vert \theta_*^{1,2}] + \ee[e_t^\tp D  e_t| \theta_*^{1,2}] \notag \\
& =  \ee^{\pi^{\diamond*}}[c(x_t^{\diamond}, u_t^{\diamond}) | \theta_*^{1,2}] + \tr (D \Sigma_t ),
\label{eq:optimal_cost_transformed}
\end{align}
where $D^2= Q^{22} + (\tilde K^2(\theta_*^2))^{\tp} R^{22} \tilde K^2(\theta_*^2)$. Note that the first equality is correct from \eqref{Model:cost}, the second equality is correct because of \eqref{eq:optimal_U_dec} and \eqref{eq:x_tplus1_transformed}, and the third equality is correct because of \eqref{eq:equality_Xs}. Finally the last equality is correct because if we define $v_t:=w_t^2$, then $e_t$ has the same dynamics as $s_t$ in Lemma \ref{lem:Sigma_convergence}, and consequently, $\cov(e_t) = \Sigma_t$.

Now, by substituting \eqref{eq:optimal_cost_transformed} in \eqref{eq:optimal_cost_dec}, considering \eqref{eq:centralized_cost_def} and the fact from Lemma \ref{lem:Sigma_convergence}, $\Sigma_t$ converges to $\Sigma$ as $t \to \infty$, the statement of the lemma follows.

\section{Proof of Lemma \ref{lem:expectation_equalities_appendix1} (Lemma \ref{lem:expectation_equalities})}
\label{proof_lem:expectation_equalities}

First note that under the policy of the \texttt{AL-SARL} algorithm, $u_t^{\diamond} = K(\theta^{1}_t, \theta_*^2) x_t^{\diamond}$ and hence, the dynamics of $x_t^{\diamond}$ in \eqref{Model:system_LQ} can be written as,
\begin{align}
x_{t+1}^{\diamond} = \Big(A_* + B_* K(\theta^{1}_t, \theta_*^2) \Big) x_t^{\diamond} + \vecc(w_t^1, \mathbf{0}).
\label{eq:X_diamond_dynamics_2}
\end{align}
Further, note that under the policy of the \texttt{AL-MARL} algorithm, 
\begin{align}
u_t = K(\theta^{1}_t, \theta_*^2) \bar x_t + \vecc(\mathbf{0}, \tilde K^2 (\theta_*^2)e_t),
\label{eq:optimal_U_dec_2}
\end{align}
where we have defined $\bar x_t := \vecc(x_t^1, \check x_t^2)$, $e_t := x_t^2 - \check x_t^2$, and we have $K(\theta^{1}_t, \theta_*^2) = \begin{bmatrix}
K^1(\theta^{1}_t, \theta_*^2)  \\ K^2(\theta^{1}_t, \theta_*^2) 
\end{bmatrix}$ from Lemma \ref{lm:opt_strategies_appendix} in the Appendix \ref{lemma1_complete}. Note that $\bar x_t $ here is different from the one in the proof of Lemma \ref{lem:connection_optimal_cost_appendix1}. Then, from the dynamics of $x_t$ in \eqref{Model:system_overall} and update equation for $\check x_t^2$ in the \texttt{AL-MARL} algorithm, we can write 
\begin{align}
x_{t+1} =  \bar x_{t+1} +  \vecc(\mathbf{0}, e_{t+1}),
\label{eq:x_tplus1_transformed_2}
\end{align}
where
\begin{align}
\label{eq:X_bar_dynamics_dec_2}
\bar x_{t+1} &= \Big( A_* + B_* K(\theta^{1}_t, \theta_*^2) \Big) \bar x_t + \vecc(w_t^1, \mathbf{0}), \quad e_{t+1} = C e_{t} + w_t^2.
\end{align}
Now by comparing \eqref{eq:X_diamond_dynamics_2} and \eqref{eq:X_bar_dynamics_dec_2} and the fact that both $x_1^{\diamond}$ and $\bar x_1$ are equal, we can see that for any time $t$,
\begin{align}
\bar x_{t+1} = x_{t+1}^{\diamond}.
\label{eq:equality_Xs_2}
\end{align}
Now, we can use the above results to write $c(x_t, u_t)$ under the \texttt{AL-MARL} algorithm as follows,
\begin{align}
c(x_t,u_t) \vert_{\texttt{AL-MARL}} &= 
\left[ x_t^\tp Q x_t + (u_t)^{\tp} R u_t \right]  \vert_{\texttt{AL-MARL}} 
= \bar x_t^\tp Q \bar x_t + (K(\theta^{1}_t, \theta_*^2) \bar x_t)^{\tp} R K(\theta^{1,2}_t) \bar x_t + e_t^{\tp} D e_t
\notag \\
&= (x_t^{\diamond})^\tp Q x_t^{\diamond} + (K(\theta^{1}_t, \theta_*^2) x_t^{\diamond})^{\tp} R K(\theta^{1}_t, \theta_*^2) x_t^{\diamond} + e_t^{\tp} D e_t  
\notag \\
&= c(x_t^{\diamond}, u_t^{\diamond})\vert_{\texttt{AL-SARL}} + e_t^{\tp} D e_t,
\end{align}
where the first equality is correct from \eqref{Model:cost}, the second equality is correct because of \eqref{eq:optimal_U_dec_2} and \eqref{eq:x_tplus1_transformed_2}, and the third equality is correct because of \eqref{eq:equality_Xs_2}.

\section{Detailed description of the \texttt{SARL} learner $\mathcal{L}$}
\label{sec:details_SARL_learner}
In this section, we describe the \texttt{SARL} learner $\mathcal{L}$ of some of existing algorithms for the \texttt{SARL} problems in details.

\subsection{TS-based algorithm of \citet{faradonbeh2017regret}}
Let $p:=d_x^1 + d_x^2 $ and $q:=d_x^1 + d_x^2 + d_u^1 + d_u^2$. Further, let $\bar K(\theta_t^{1,2}) := 
\begin{bmatrix}
\mathbf{I}_p \\ K(\theta_t^{1,2})
\end{bmatrix} \in \R^{q \times p}$.
\begin{figure}[H]
\begin{center}
\begin{tikzpicture}
\node [rectangle,draw,minimum width=1.2cm,minimum height=0.75cm,line width=1pt,rounded corners, fill=black!25]at (-1.0,0) (1) {
\begin{small}
\begin{tabular}{l}
\textbf{Initialize}
\\
\underline{Parameters}: $V_0 \in \R^{q \times q}$ (a PD matrix), $\mu_0 \in \R^{p \times q}$, $\eta >1$ (reinforcement rate)
\\
\rule{10cm}{0.6pt}
\\
\textbf{Input/Output} \\
\\
\begin{varwidth}{\linewidth}
            \begin{algorithmic}
            \State \underline{Input}: $t$ and $x_t^{\diamond}$ 
            \\
            \If{$t = \lfloor \eta^m \rfloor$ for some $m=0, 1, \ldots$}
                \State Sample $\hat \theta_t$ from a Gaussian with mean $\mu_m$ and covariance $V_m^{-1}$ 
                \\
                \State \quad \quad $\mu_m = \argmin_{\mu} \sum_{s=0}^{t-1} \norm{ x_{s+1}^{\diamond} - \mu \bar K(\theta_t^{1,2}) x_s^{\diamond}  }_2^2$
                \\
                \State \quad \quad $V_m = V_0 + \sum_{s=0}^{t-1} \bar K(\theta_t^{1,2}) x_s^{\diamond} (x_s^{\diamond})^{\tp} \bar K(\theta_t^{1,2})^{\tp}$
                \Else
                \State $\hat \theta_t = \hat \theta_{t-1}$
            \EndIf
             \State Partition $\hat \theta_t$ to find $\theta_t^1 = [A_t^1, B_t^1]$ and $\theta_t^2 = [A_t^2, B_t^2]$
                \\
                \State \quad \quad $\hat \theta_t = \begin{bmatrix}
				A_t^1 & \mathbf{\times} & B_t^1 & \mathbf{\times} \\
				 \mathbf{\times} & A_t^2 &  \mathbf{\times} & B_t^2
				\end{bmatrix}                
                $ \\
            \State Calculate $K(\theta_t^{1,2})$ and store $x_t^{\diamond}$ and $K(\theta_t^{1,2})$ for next steps
            \\
            \State \underline{Output}: $\theta_t^{1}$ and $\theta_t^{2}$
        \end{algorithmic}%
        \end{varwidth}
\end{tabular}
\end{small}}; 
\path[thick,->,>=stealth, dashed, line width=1pt]
	($(1.north) + (0,0.5)$) edge node{} ($(1.north) + (0,0)$)
	;
\path[thick,->,>=stealth,line width=1pt]
	($(1.west) + (-0.5,0)$) edge node{} ($(1.west) + (0,0)$)
	($(1.east) + (0,0)$)	 edge node{} ($(1.east) + (0.5,0)$)
	;

\node[] at ($(1.north) + (0,0.7)$) {Initialize parameters};
\node[] at ($(1.west) + (-1.2,0.25)$) {state $x_t^{\diamond}$};
\node[] at ($(1.west) + (-1.2,-0.25)$) {time $t$};
\node[] at ($(1.east) + (1.5,0.25)$) {$\theta_t^1 = [A_t^1, B_t^1]$};
\node[] at ($(1.east) + (1.5,-0.25)$) {$\theta_t^2 = [A_t^2, B_t^2]$};

\end{tikzpicture}
\caption{\texttt{SARL} learner of TS-based algorithm of \citet{faradonbeh2017regret}}
\label{fig:TS_faradonbeh}
\end{center}
\end{figure}

\subsection{TS-based algorithm of \citet{abbasi2015bayesian}}
Let $p:=d_x^1 + d_x^2 $ and $q:=d_x^1 + d_x^2 + d_u^1 + d_u^2$. 
\begin{figure}[H]
\begin{center}
\begin{tikzpicture}
\node [rectangle,draw,minimum width=1.2cm,minimum height=0.75cm,line width=1pt,rounded corners, fill=black!25]at (-1.0,0) (1) {
\begin{small}
\begin{tabular}{l}
\textbf{Initialize}
\\
\underline{Parameters}: $V_0 \in \R^{q \times q}$ (a PD matrix), $\mu_0 \in \R^{p \times q}$, $\eta >1$ (reinforcement rate)
\\
\rule{10cm}{0.6pt}
\\
\textbf{Input/Output} \\
\\
\begin{varwidth}{\linewidth}
            \begin{algorithmic}        
             \State \underline{Input}: $t$ and $x_t^{\diamond}$ 
             \\   
              \If{$t>0$}
               \State Update $P_{t-1}$ with $(x_{t-1}^{\diamond}, u_{t-1}^{\diamond}, x_{t}^{\diamond})$ to obtain $P_t$
               \State $V_t = V_{t-1} + \vecc(x_{t-1}^{\diamond}, u_{t-1}^{\diamond}) \vecc(x_{t-1}^{\diamond}, u_{t-1}^{\diamond})^{\tp}$
              \EndIf
            \If{$\det(V_t) > g \det(V_{last})$ or $t=0$}
                \State Sample $\hat \theta_t$ from $P_t$
                \State $V_{last} = V_t$
                \Else
                \State $\hat \theta_t = \hat \theta_{t-1}$
            \EndIf
             \State Partition $\hat \theta_t$ to find $\theta_t^1 = [A_t^1, B_t^1]$ and $\theta_t^2 = [A_t^2, B_t^2]$
                \\
                \State \quad \quad $\hat \theta_t = \begin{bmatrix}
				A_t^1 & \mathbf{\times} & B_t^1 & \mathbf{\times} \\
				 \mathbf{\times} & A_t^2 &  \mathbf{\times} & B_t^2
				\end{bmatrix}                
                $ \\
            \State Calculate $K(\theta_t^{1,2})$ and store $x_t^{\diamond}$ and $u_t^{\diamond} = K(\theta_t^{1,2})x_t^{\diamond}$ for next steps
            \\
            \State \underline{Output}: $\theta_t^{1}$ and $\theta_t^{2}$
        \end{algorithmic}%
        \end{varwidth}
\end{tabular}
\end{small}}; 
\path[thick,->,>=stealth, dashed, line width=1pt]
	($(1.north) + (0,0.5)$) edge node{} ($(1.north) + (0,0)$)
	;
\path[thick,->,>=stealth,line width=1pt]
	($(1.west) + (-0.5,0)$) edge node{} ($(1.west) + (0,0)$)
	($(1.east) + (0,0)$)	 edge node{} ($(1.east) + (0.5,0)$)
	;

\node[] at ($(1.north) + (0,0.7)$) {Initialize parameters};
\node[] at ($(1.west) + (-1.2,0.25)$) {state $x_t^{\diamond}$};
\node[] at ($(1.west) + (-1.2,-0.25)$) {time $t$};
\node[] at ($(1.east) + (1.5,0.25)$) {$\theta_t^1 = [A_t^1, B_t^1]$};
\node[] at ($(1.east) + (1.5,-0.25)$) {$\theta_t^2 = [A_t^2, B_t^2]$};

\end{tikzpicture}
\caption{\texttt{SARL} learner of TS-based algorithm of \citet{abbasi2015bayesian}}
\label{fig:TS_faradonbeh}
\end{center}
\end{figure}

\section{Optimal strategies for the optimal multi-agent LQ problem and the optimal single-agent LQ problem}
\label{lemma1_complete}
In this section, we provide the following two lemmas. The first lemma (Lemma \ref{lm:opt_strategies_appendix}) is the complete version of Lemma \ref{lm:opt_strategies} which describes optimal strategies for the optimal multi-agent LQ problem of Section \ref{sec:optimal_dec_LQ}. The second lemma (Lemma \ref{lm:opt_strategies_centralized_appendix}) describes optimal strategies for the optimal single-agent LQ problem of Section \ref{sec:centralized_LQ}.
\begin{lemma}[\citet{ouyang2018optimal}, complete version of Lemma \ref{lm:opt_strategies}]
\label{lm:opt_strategies_appendix}
Under Assumption \ref{assum:det_stb}, the optimal infinite horizon cost $J(\theta_*^{1,2}) $ is given by
\begin{align}
J(\theta_*^{1,2}) = & \sum_{n=1}^2 \tr \Big( \gamma^n P^{nn}(\theta_*^{1,2}) + (1-\gamma^n) \tilde P^n (\theta_*^n) \Big),
\label{eq:opcost_infinite_appendix}
\end{align}
where $P(\theta_*^{1,2}) = [P^{\cdot, \cdot}(\theta_*^{1,2})]_{1,2}$, $\tilde P^1(\theta_*^{1})$, and $\tilde P^2(\theta_*^{2})$ are the unique PSD solutions to the following Ricatti equations:
\begin{align}
P(\theta_*^{1,2}) &=  \mathcal{R} (P(\theta_*^{1,2}),Q,R,A_*,B_*),
\label{eq:P_infinite_appendix}
\\
%&\tilde P^n(\theta_*^n) =  \mathcal{R} ( \tilde P^n(\theta_*^{n}) , Q^{nn},R^{nn},A_*^{n},B_*^{n}), \quad n \in \{1,2\}.
%\\
\tilde P^1(\theta_*^1) &=  \mathcal{R} ( \tilde P^1(\theta_*^{1}) , Q^{11},R^{11},A_*^{1},B_*^{1}), 
\quad 
\tilde P^2(\theta_*^2) =  \mathcal{R} ( \tilde P^2(\theta_*^{2}) , Q^{22},R^{22},A_*^{2},B_*^{2}).
\label{eq:tildeP_infinite_appendix}
\end{align}
The optimal control strategies are given by
\begin{align}
u^{1}_t = K^1(\theta_*^{1,2})  \begin{bmatrix}
\hat x_t^1 \\
\hat x_t^2 
\end{bmatrix}+ \tilde K^1 (\theta_*^{1}) (x_t^1 - \hat x_t^1),
\quad 
u^{2}_t = K^2(\theta_*^{1,2})  \begin{bmatrix}
\hat x_t^1 \\
\hat x_t^2 
\end{bmatrix} + \tilde K^2 (\theta_*^{2}) (x_t^2 - \hat x_t^2),
 \label{eq:opt_U_appendix}
\end{align}
where the gain matrices $K(\theta_*^{1,2}) := \begin{bmatrix}
K^1(\theta_*^{1,2}) \\ K^2(\theta_*^{1,2})
\end{bmatrix}$, $\tilde K^1(\theta_*^{1})$, and $\tilde K^2(\theta_*^{2})$ are given by
\begin{align}
K(\theta_*^{1,2}) &= \mathcal{K} (P(\theta_*^{1,2}),R,A_*,B_*),
\label{eq:K_infinite_appendix}
%\\
%& \tilde K^n(\theta_*^{n}) = \mathcal{K} (\tilde P^n (\theta_*^{n}) ,R^{nn},A_*^{n},B_*^{n}),\quad n \in \{1,2\}.
\\
\tilde K^1(\theta_*^{1}) &= \mathcal{K} (\tilde P^1 (\theta_*^{1}) ,R^{11},A_*^{1},B_*^{1}),
\quad
\tilde K^2(\theta_*^{2}) = \mathcal{K} (\tilde P^1 (\theta_*^{2}) ,R^{22},A_*^{2},B_*^{2}).
\label{eq:tildeK_infinite_appendix}
\end{align}
Furthermore $\hat x_t^n = \ee[x_t^n|h^c_t, \theta_*^{1,2}]$, $n \in \{1,2\}$, is the estimate (conditional expectation) of $x_t^n$ based on the information $h^c_t$ which is common among the agents, that is, $h_t^c = h_t^1 \cap h_t^2$.  The estimates $\hat x_t^n$, $n \in \{1,2\}$, can be computed recursively according to
\begin{align}
&\hat x_0^n =  x_0^n, 
\quad \quad 
\hat x_{t+1}^n=  %\notag \\  
%&
\begin{cases}
    x_{t+1}^n     & \text{if } \gamma^n=1    \\
     A_*^{n} \hat x_t^n + B_*^{n} K^{n}(\theta_*^{1,2}) \vecc(\hat x_t^1, \hat x_t^2)   & \text{otherwise}
  \end{cases}.
\label{eq:estimator_t_infinite_appendix}
\end{align}

\end{lemma}

\begin{lemma}[\citet{kumar2015stochastic,bertsekas1995dynamic}]
\label{lm:opt_strategies_centralized_appendix}
Under Assumption \ref{assum:det_stb}, the optimal infinite horizon cost $J^{\diamond}(\theta_*^{1,2})$ is given by $J^{\diamond}(\theta_*^{1,2})  = \tr ([P(\theta_*^{1,2})]_{1,1})$ where $P(\theta_*^{1,2})$ is as defined in \eqref{eq:P_infinite_appendix}. Furthermore, the optimal strategy $\pi^{\diamond *}$ is given by $u_t^{\diamond} = K(\theta_*^{1,2}) x_t^{\diamond}$ where $K(\theta_*^{1,2})$ is as defined in \eqref{eq:K_infinite_appendix}.
\end{lemma}

\section{Details of Experiments}
\label{sec:experiments_appendix}
In this section, we illustrate the performance of the \texttt{AL-MARL} algorithm through numerical experiments. Our proposed algorithm requires a \texttt{SARL} learner. As the TS-based algorithm of \citet{faradonbeh2017regret} achieves a $\tilde O(\sqrt{T})$ regret for a \texttt{SARL} problem, we use the \texttt{SARL} learner of this algorithm (The details for this \texttt{SARL} learner are presented in Appendix \ref{sec:details_SARL_learner}).

%OFU-based algorithms of \citep{faradonbeh2017regret, faradonbeh2019applications} and the TS-based algorithm of \citep{faradonbeh2017regret} 

%Our proposed algorithm requires a sampling schedules $\mathcal{S}$ and a sampling condition $\mathcal{C}$. We use the sampling schedule $\mathcal{S}$ of \citep{ouyang2017learning, faradonbeh2018optimality, abbasi2015bayesian} described in Remarks 1-3 and the sampling condition $\mathcal{C}$ of \citep{ouyang2017learning} described in Remark 4 (The details for the parameters of the sampling schedules $\mathcal{S}$ of  \citep{ouyang2017learning, faradonbeh2018optimality, abbasi2015bayesian} are presented in Appendix X).

We consider an instance of the \texttt{MARL2} problem where system 1 (which is unknown to the agents), system 2 (which is known to the agents), and matrices of the cost funtion have the following parameters ({note that the unknown system and the strcuture of the matrices of the cost function are similar to the model} studied in \citet{tu2017least, abbasi2018regret, dean2017sample}) with $d_x^1= d_x^2 = d_u^1 = d_u^2 = 3$, 
\begin{align}
A^{11} = \begin{bmatrix}
1.01 & 0.01 & 0 \\
0.01 & 1.01 & 0.01 \\
0 & 0.01 & 1.01
\end{bmatrix}, \quad A^{22} = B^{11} = B^{22} = \mathbf{I}_3, \quad Q = 10^{-3} \mathbf{I}_6, \quad R = \mathbf{I}_6.
\end{align}
{We use the following parameters for the TS-based learner $\mathcal{L}$ of \citet{faradonbeh2017regret} as described in Appendix \ref{sec:details_SARL_learner}:
$V_0$ to be a $12 \times 12$ identity matrix, $\mu_0$ to be a $6 \times 12$ zero matrix, $\eta$ to be equal to 1.1.}

The theoretical result of Theorem \ref{thm:regret_bound_case2} holds when Assumption \ref{assum:seed} is true. Since we use the TS-based learner of \citet{faradonbeh2017regret}, this assumption can be satisfied by setting the same sampling seed between the agents. Here, we consider both cases of same sampling seed and arbitrary sampling seed for the experiments. We ran 100 simulations and show the mean of regret with the 95\% confidence interval for each scenario.

As it can be seen from Figure \ref{fig:MARL_1}, for both of theses cases, our proposed algorithm with the TS-based learner $\mathcal{L}$ of \citet{faradonbeh2017regret} achieves a $\tilde O(\sqrt{T})$ regret for our \texttt{MARL2} problem, which matches the theoretical results of Corollary \ref{cor:regret_bound_case}.

%\begin{figure}
% \begin{center}
%\includegraphics[width=8cm,height=6cm]{TSDRL_Faradonbeh_nips4.pdf}
%\end{center}
%\caption{\texttt{AL-MARL} algorithm with the \texttt{SARL} learner of \citep{faradonbeh2017regret}}
%\label{fig:MARL_1_appendix}
%\end{figure}

\section{Proof of Theorem \ref{lem:regret_bound_extension}}
\label{proof:lem:regret_bound_extension}
We prove this theorem in the case where $N_1=2$ systems (systems 1 and 2) are unknown and $N -N_1 = 2$ systems (systems 3 and 4) are known for ease of presentation. The case with general $N$ and $N_1$ will follow by the same arguments. We assume that there exist communication links from agents 1 and 2 to all other agents, but there is no communication link from agents 3 and 4 to the other agents. Let $[N]:= \{1,2,3,4\}$.

In this problem, the linear dynamics of system $n \in [N]$ are given by
\begin{align}
x_{t+1}^n = A_*^{n} x_t^n + B_*^{n} u_t^n + w_t^n
\label{Model:dynamics_ext}
\end{align}
where for $n \in \mathcal{N}$, $x_t^n \in \R^{d_x^n}$ is the state of system $n$ and $u_t^n \in \R^{d_u^n}$ is the action of agent $n$. $A_*^{n}$ and $B_*^{n}$, $n \in [N]$, are system matrices with appropriate dimensions. We assume that for $n \in [N]$, $w_t^n$, $t \geq0$, are i.i.d with standard Gaussian distribution $\mathcal{N}(\mathbf{0}, \mathbf{I})$. The initial states $x_0^{n}$, $n \in [N]$, are assumed to be fixed and known.

The overall system dynamics can be written as,
\begin{align}
&x_{t+1} = A_* x_t + B_* u_t + w_t,
\label{Model:system_overall_ext}
\end{align}
where we have defined $x_t = \vecc(x_t^1,\ldots, x_t^4), u_t = \vecc(u_t^1, \ldots, u_t^4), w_t = \vecc(w_t^1, \ldots, w_t^4)$, $A_* = \textbf{diag}(A_*^1, \ldots, A_*^4)$, and $B_* = \textbf{diag}(B_*^1,\ldots, B_*^4)$.

At each time $t$, agent $n$'s action is a function $\pi_t^n$ of its information $h_t^n$, that is, $u_t^n = \pi_t^n(h_t^n)$ where $h_t^1 = \{x_{0:t}^1, u_{0:t-1}^1\}$, $h_t^2 = \{x_{0:t}^2, u_{0:t-1}^2\}$, $h_t^3 = \{x_{0:t}^3, u_{0:t-1}^3, x_{0:t}^{1,2}\}$, and $h_t^3 = \{x_{0:t}^4, u_{0:t-1}^4, x_{0:t}^{1,2}\}$. Let $\pi = (\pi^1, \ldots, \pi^4)$ where $\pi^n = (\pi_0^n, \pi_1^n, \ldots)$. 

At time $t$, the system incurs an instantaneous cost $c(x_t,u_t)$, which is a quadratic function given by
\begin{align}
c(x_t,u_t) &=  x_t^\tp Q x_t + u_t^\tp R u_t, 
\label{Model:cost_ext}
\end{align}
where $Q = [Q^{\cdot, \cdot}]_{1:4}$ is a known symmetric positive semi-definite (PSD) matrix and $R = [R^{\cdot, \cdot}]_{1:4}$ is a known symmetric positive definite (PD) matrix. 

\subsection{The optimal multi-agent linear-quadratic problem}
\label{sec:optimal_dec_LQ_ext}
Let $\theta_*^n = [A_*^n,B_*^n]$ be the dynamics parameter of system $n$, $n \in [N]$. % and  $\theta_* = (\theta_*^1, \theta_*^2)$. 
When $\theta_*^{1:4}$ are perfectly known to the agents, minimizing the infinite horizon average cost  is a multi-agent stochastic Linear Quadratic (LQ) control problem. Let $J(\theta_*^{1:4})$ be the optimal infinite horizon average cost under $\theta_*^{1:4}$, that is,
\begin{align}
J(\theta_*^{1:4}) = \inf_{\pi} \limsup_{T\rightarrow\infty} \frac{1}{T} \sum_{t=0}^{T-1} \ee^{\pi}[c(x_t,u_t) | \theta_*^{1:4}].
\label{eq:optimal_cost_decentralized_ext}
\end{align}

The above decentralized stochastic LQ problem has been studied by \cite{ouyang2018optimal}. The following lemma summarizes this result.
\begin{lemma}
\label{lm:opt_strategies_ext}
Under Assumption \ref{assum:det_stb}, the optimal infinite horizon cost $J(\theta_*^{1:4}) $ is given by
\begin{align}
J(\theta_*^{1:4}) = & \sum_{n=1}^2 \tr \Big(P^{nn}(\theta_*^{1:4})  \Big) + 
\sum_{n=3}^4 \tr \Big( \tilde P^n (\theta_*^n)  \Big),
\label{eq:opcost_infinite_ext}
\end{align}
where $P(\theta_*^{1:4}) = [P^{\cdot, \cdot}(\theta_*^{1:4})]_{1:4}$, $\tilde P^3(\theta_*^{3})$, and $\tilde P^4(\theta_*^{4})$ are the unique PSD solutions to the following Ricatti equations:
\begin{align}
P(\theta_*^{1:4}) &=  \mathcal{R} (P(\theta_*^{1:4}),Q,R,A_*,B_*),
\label{eq:P_infinite_ext}
\\
%&\tilde P^n(\theta_*^n) =  \mathcal{R} ( \tilde P^n(\theta_*^{n}) , Q^{nn},R^{nn},A_*^{n},B_*^{n}), \quad n \in \{1,2\}.
%\\
\tilde P^3(\theta_*^4) &=  \mathcal{R} ( \tilde P^3(\theta_*^{3}) , Q^{33},R^{33},A_*^{3},B_*^{3}), 
\quad 
\tilde P^4(\theta_*^4) =  \mathcal{R} ( \tilde P^4(\theta_*^{4}) , Q^{44},R^{44},A_*^{4},B_*^{4}).
\label{eq:tildeP_infinite_ext}
\end{align}
The optimal control strategies are given by
\begin{align}
u^{1}_t &= K^1(\theta_*^{1:4})  \begin{bmatrix}
x_t^1 \\
x_t^2 \\
\hat x_t^3 \\
\hat x_t^4 
\end{bmatrix},
\quad 
u^{2}_t = K^2(\theta_*^{1:4})  \begin{bmatrix}
x_t^1 \\
x_t^2 \\
\hat x_t^3 \\
\hat x_t^4 
\end{bmatrix}, \notag \\
u^{3}_t &= K^3(\theta_*^{1:4})  \begin{bmatrix}
x_t^1 \\
x_t^2 \\
\hat x_t^3 \\
\hat x_t^4 
\end{bmatrix}+ \tilde K^3 (\theta_*^{3}) (x_t^3 - \hat x_t^3),
\quad 
u^{4}_t &= K^4(\theta_*^{1:4})  \begin{bmatrix}
x_t^1 \\
x_t^2 \\
\hat x_t^3 \\
\hat x_t^4 
\end{bmatrix}+ \tilde K^4 (\theta_*^{4}) (x_t^4 - \hat x_t^4),
 \label{eq:opt_U_ext}
\end{align}
%\begin{align}
%\bmat{u^{1}_t \\ u^{2}_t \\} = 
%\begin{bmatrix}
%K^1(\theta_*^{1,2}) \\ K^2(\theta_*^{1,2})
%\end{bmatrix}
% \begin{bmatrix}
%\hat x_t^1 \\
%\hat x_t^2 
%\end{bmatrix}
%+ \bmat{\tilde K^1 (\theta_*^{1}) (x_t^1 - \hat x_t^1) \\ \tilde K^2 (\theta_*^{2}) (x_t^2 - \hat x_t^2)}, \label{eq:opt_U}
%\end{align}
where the gain matrices $K(\theta_*^{1:4}) := \begin{bmatrix}
K^1(\theta_*^{1:4}) \\ K^2(\theta_*^{1:4}) \\ K^3(\theta_*^{1:4}) \\ K^4(\theta_*^{1:4})
\end{bmatrix}$, $\tilde K^3(\theta_*^{3})$, and $\tilde K^4(\theta_*^{4})$ are given by
\begin{align}
K(\theta_*^{1:4}) &= \mathcal{K} (P(\theta_*^{1:4}),R,A_*,B_*),
\label{eq:K_infinite_ext}
\\
\tilde K^3(\theta_*^{3}) &= \mathcal{K} (\tilde P^3 (\theta_*^{3}) ,R^{33},A_*^{3},B_*^{3}),
\quad
\tilde K^4(\theta_*^{4}) = \mathcal{K} (\tilde P^4 (\theta_*^{4}) ,R^{44},A_*^{4},B_*^{4}).
\label{eq:tildeK_infinite_ext}
\end{align}
Furthermore $\hat x_t^n = \ee[x_t^n|h^c_t, \theta_*^{1:4}]$, $n \in \{3,4\}$, is the estimate (conditional expectation) of $x_t^n$ based on the information $h^c_t$ which is common among all the agents.  The estimates $\hat x_t^n$, $n \in \{3,4\}$, can be computed recursively according to
\begin{align}
&\hat x_0^n =  x_0^n, 
\quad \quad 
\hat x_{t+1}^n=  A_*^{n} \hat x_t^n + B_*^{n} K^{n}(\theta_*^{1:4}) \vecc(x_t^1,x_t^2, \hat x_t^3, \hat x_t^4).
\label{eq:estimator_t_infinite_ext}
\end{align}

\end{lemma}

\subsection{The multi-agent reinforcement learning problem}
The problem we are interested in is to minimize the infinite horizon average cost when the system parameters $\theta_*^1 = [A_*^{1}, B_*^1]$ and $\theta_*^2 = [A_*^{2}, B_*^2]$ are unknown and $\theta_*^3 = [A_*^{3}, B_*^3]$ and $\theta_*^4 = [A_*^{4}, B_*^4]$ are known. For future reference, we call this problem \texttt{MARL3}. In this case, the learning performance of policy $\pi$ is measured by the cumulative regret over $T$ steps defined as,
\begin{align}
R(T, \pi) = \sum_{t=0}^{T-1} \left[ c(x_t, u_t) - J(\theta_*^{1:4}) \right].
\label{eq:regret_decentralized_ext}
\end{align} 

\subsection{A single-agent LQ problem}
\label{sec:centralized_LQ_ext}
In this section, we construct an auxiliary single-agent LQ control problem. This auxiliary single-agent LQ control problem will be used later as a coordination mechanism for our \texttt{MARL} algorithm.
%This auxiliary single-agent LQ control problem will be used later for the regret analysis of the \texttt{MARL3} problem.

Consider a single-agent system with dynamics 
\begin{align}
x_{t+1}^{\diamond} = A_* x_t^{\diamond} + B_* u_t^{\diamond} + \begin{bmatrix}
w_t^1\\ w_t^2 \\ \mathbf{0} \\ \mathbf{0}
\end{bmatrix},
\label{Model:system_LQ_ext}
\end{align}
where $x_t^{\diamond} \in \R^{d_x^1 + d_x^2 + d_x^3 + d_x^4}$ is the state of the system, $u_t^{\diamond} \in  \R^{d_u^1 + d_u^2 + d_u^3 + d_u^4}$ is the action of the auxiliary agent, $w_t^n$, $n \in \{1,2\}$, is the noise vector of system $n$ defined in \eqref{Model:dynamics_ext}, and matrices $A_*$ and $B_*$ are as defined in \eqref{Model:system_overall_ext}. The initial state $x_0^{\diamond}$ is assumed to be equal to $x_0$. The action $u_t^{\diamond} = \pi_t^{\diamond} (h_t^{\diamond})$ at time $t$ is a function of the history of observations $h_t^{\diamond} = \{x_{0:t}^{\diamond}, u_{0:t-1}^{\diamond}\}$. The auxiliary agent's strategy is denoted by $\pi^{\diamond} = (\pi_1^{\diamond}, \pi_2^{\diamond},\ldots)$. The instantaneous cost $c(x_t^{\diamond},u_t^{\diamond})$ of the system is a quadratic function given by
\begin{align}
c(x_t^{\diamond}, u_t^{\diamond}) &= 
(x_t^{\diamond})^\tp Q x_t^{\diamond} + (u_t^{\diamond})^\tp R u_t^{\diamond},
\label{Model:cost_LQ_ext}
\end{align}
where matrices $Q$ and $R$ are as defined in \eqref{Model:cost_ext}.

When $\theta_*^{1:4}$ are known to the auxiliary agent, minimizing the infinite horizon average cost is a single-agent stochastic Linear-Quadratic (LQ) control problem. Let $J^{\diamond}(\theta_*^{1:4})$ be the optimal infinite horizon average cost under $\theta_*^{1:4}$, that is,
\begin{align}
J^{\diamond}(\theta_*^{1:4}) = \inf_{\pi^{\diamond}} \limsup_{T\rightarrow\infty} \frac{1}{T} \sum_{t=0}^{T-1} \ee^{\pi^{\diamond}}[c(x_t^{\diamond}, u_t^{\diamond}) | \theta_*^{1:4}].
\label{eq:optimal_cost_centralized_ext}
\end{align}
%The above centralized stochastic LQ problem has been widely studied in the literature \cite{}. The following lemma summarizes this result.
Then, the following lemma summarizes the result for the optimal infinite horizon single-agent LQ control problem.
\begin{lemma}[\cite{kumar2015stochastic,bertsekas1995dynamic}]
\label{lm:opt_strategies_centralized_ext}
Under Assumption \ref{assum:det_stb}, the optimal infinite horizon cost $J^{\diamond}(\theta_*^{1:4})$ is given by $J^{\diamond}(\theta_*^{1,2})  = \tr ([P(\theta_*^{1:4})]_{1,1}) + \tr ([P(\theta_*^{1:4})]_{2,2})$ where $P(\theta_*^{1:4})$ is as defined in \eqref{eq:P_infinite_ext}. Furthermore, the optimal strategy $\pi^{\diamond *}$ is given by $u_t^{\diamond} = K(\theta_*^{1:4}) x_t^{\diamond}$ where $K(\theta_*^{1:4})$ is as defined in \eqref{eq:K_infinite_ext}.
\end{lemma}
When the actual parameters $\theta_*^{1:4}$ are unknown, this single-agent stochastic LQ control problem becomes a Single-Agent Reinforcement Learning (\texttt{SARL2}) problem. We define the regret of a policy $\pi^{\diamond}$ over $T$ steps compared with the optimal infinite horizon cost $J^{\diamond}(\theta_*^{1:4})$ to be
\begin{align}
R^{\diamond}(T, \pi^{\diamond}) =  \sum_{t=0}^{T-1} \left[ c(x_t^{\diamond}, u_t^{\diamond}) - J^{\diamond}(\theta_*^{1:4})  \right].
\label{eq:regret_centralized_ext}
\end{align} 

Similar to Section \ref{sec:centralized_LQ}, we can generally describe the existing proposed algorithms for the \texttt{SARL2} problem as \texttt{AL-SARL2} algorithm with the \texttt{SARL2} learner $\mathcal{L}$ of Figure \ref{fig:SystemModel_ext}.

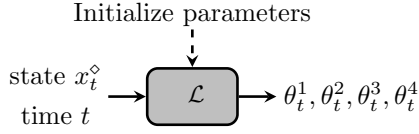
\begin{figure}
\begin{center}
\begin{tikzpicture}
\node [rectangle,draw,minimum width=1.2cm,minimum height=0.75cm,line width=1pt,rounded corners, fill=black!25]at (-1.0,0) (1) {
\begin{small}
\begin{tabular}{c}
$\mathcal{L}$
\end{tabular}
\end{small}}; 
\path[thick,->,>=stealth, dashed, line width=1pt]
	($(1.north) + (0,0.5)$) edge node{} ($(1.north) + (0,0)$)
	%($(1.south) + (0,-0.5)$) edge node{} ($(1.south) + (0,0)$)
	;
\path[thick,->,>=stealth,line width=1pt]
	($(1.west) + (-0.5,0)$) edge node{} ($(1.west) + (0,0)$)
	($(1.east) + (0,0)$)	 edge node{} ($(1.east) + (0.5,0)$)
	;

\node[] at ($(1.north) + (0,0.7)$) {Initialize parameters};
%\node[] at ($(1.south) + (0,-0.7)$) {Update};
\node[] at ($(1.west) + (-1.2,0.25)$) {state $x_t^{\diamond}$};
\node[] at ($(1.west) + (-1.2,-0.25)$) {time $t$};
\node[] at ($(1.east) + (1.5,0)$) {$\theta_t^1, \theta_t^2, \theta_t^3, \theta_t^4$};
%\node[] at ($(1.east) + (1.5,-0.25)$) {$\theta_t^2 = [A_t^2, B_t^2]$};
%\node[] at ($(1.east) + (1.5,0.25)$) {$\theta_t^1 = [A_t^1, B_t^1]$};
%\node[] at ($(1.east) + (1.5,-0.25)$) {$\theta_t^2 = [A_t^2, B_t^2]$};

\end{tikzpicture}
\caption{\texttt{SARL2} learner as a block box.}
\label{fig:SystemModel_ext}
\end{center}
\end{figure}

\begin{algorithm}[tb]
   \caption{\texttt{AL-SARL2}} %\cite{ouyang2017learning}}
   \label{alg:TS_centralized_ext}
\begin{algorithmic}
  \State Initialize $\mathcal{L}$ and $x_0^{\diamond}$
   \For{$t=0,1,\ldots$}
\State Feed time $t$ and state $x_t^{\diamond}$ to $\mathcal{L}$ and get $\theta_t^{1:4}$
   \State Compute $K(\theta_t^{1:4})$ from \eqref{eq:K_infinite_ext} and execute $u_t^{\diamond} = K(\theta_t^{1:4}) x_t^{\diamond}$
  % \Endif
   \State Observe new state $x_{t+1}^{\diamond}$
  % \State Update $\mathcal{L}$
   \EndFor
\end{algorithmic}
\end{algorithm}

\subsection{An algorithm for the \texttt{MARL3} problem}
In this Section, we propose the \texttt{AL-MARL2} algorithm based on the \texttt{AL-SARL2} algorithm. 
\begin{algorithm}[H]
   \caption{\texttt{AL-MARL2}} %\cite{ouyang2017learning}}
   \label{alg:TS_decentralized_ext}
\begin{algorithmic}
\State \textbf{Input:} \texttt{agent_ID}, $x_0^1$, $x_0^2$, $x_0^3$, and $x_0^4$
  \State Initialize $\mathcal{L}$, $\check x_0^3 = x_0^3$ and $\check x_0^4 = x_0^4$
   \For{$t=0,1,\ldots$}
\State Feed time $t$ and state $\vecc(x_t^1, x_t^2, \check x_t^3, \check x_t^4)$ to $\mathcal{L}$ and get $\theta_t^{1:4}$
 \State Compute $K^{\texttt{agent_ID}}(\theta_t^1, \theta_t^2, \theta_*^3, \theta_*^4)$
   \If{$\texttt{agent_ID}=1,2$}
   \State Execute $u_t^{\texttt{agent_ID}}= K^{\texttt{agent_ID}}(\theta_t^1, \theta_t^2, \theta_*^3, \theta_*^4) \vecc(x_t^1, x_t^2, \check x_t^3, \check x_t^4)$
   \Else
  \State Execute $u_t^{\texttt{agent_ID}}= K^{\texttt{agent_ID}}(\theta_t^1, \theta_t^2, \theta_*^3, \theta_*^4) \vecc(x_t^1, x_t^2, \check x_t^3, \check x_t^4)$ 
  \State \hspace{26mm} $+ \tilde K^{\texttt{agent_ID}}(\theta_*^{\texttt{agent_ID}}) (x_t^{\texttt{agent_ID}} - \check x_t^{\texttt{agent_ID}})$ 
   \EndIf
   \State  Observe new states $x_{t+1}^1$ and $x_{t+1}^2$ 
   \State Compute $\check x_{t+1}^3 = A_*^3 \check x_t^3 + B_*^3  K^3(\theta_t^1, \theta_t^2, \theta_*^3, \theta_*^4) \vecc(x_t^1, x_t^2, \check x_t^3, \check x_t^4) $
   \State Compute $\check x_{t+1}^4 = A_*^4 \check x_t^4 + B_*^4  K^4(\theta_t^1, \theta_t^2, \theta_*^3, \theta_*^4) \vecc(x_t^1, x_t^2, \check x_t^3, \check x_t^4) $
     \If{$\texttt{agent_ID}=3$}
     \State Observe new state $x_{t+1}^3$
     \EndIf
     \If{$\texttt{agent_ID}=4$}
     \State Observe new state $x_{t+1}^4$
     \EndIf
    % \State Update $\mathcal{L}$
%     \State Compute $\check x_{t+1}^2 = A_*^2 \check x_t^2 + B_*^2 K^2(\theta_t^1, \theta_*^2) \vecc(x_t^1, \check x_t^2)  $
   \EndFor
\end{algorithmic}
\end{algorithm}

\subsection{The regret bound for the \texttt{AL-MARL2} algorithm}
As in the proof of Theorem \ref{thm:regret_bound_case2} in Appendix \ref{sec:proof_theorem_2}, we first state some preliminary results in the following lemmas which will be used in the proof of Theorem \ref{thm:regret_bound_case2}. {Note that these lemmas are essentially the same as Lemma \ref{lem:Sigma_convergence} and Lemma \ref{thm:prob_regret}. They  have been rewritten below to be compatible with the notation of the \texttt{MARL3} problem. }

\begin{lemma}
\label{lem:Sigma_convergence_appendix}
Let $s_t^n$ be a random process that evolves as follows,
\begin{align}
s_{t+1}^n = C^n s_t^n + v_t^n, \quad s_0 = 0,
\label{eq:s_t_process_appendix} 
\end{align}
where $v_t^n$, $t \geq 0$, are independent Gaussian random vectors with zero-mean and covariance matrix $\cov(v_t^n) = \mathbf{I}$. Further, let $C^n=A_*^{n} + B_*^{n} \tilde K^n(\theta_*^n)$ and define $\Sigma_t^n = \cov(s_t^n)$, then the sequence of matrices $\Sigma_t^n$, $t \geq0$, is increasing and it converges to a PSD matrix $\Sigma^n$ as $t \to \infty$. Further, $C^n$ is a stable matrix, that is, $\rho(C^n) < 1$.
\end{lemma}
%
%\begin{proof}
%See Appendix \ref{proof_lem:Sigma_convergence} for a proof.
%\end{proof}
%
\begin{lemma}
\label{thm:prob_regret_appendix}
Let $s_t^n$ be a random process defined as in Lemma \ref{lem:Sigma_convergence_appendix}. Let $D^n$ be a positive semi-definite (PSD) matrix. Then for any $\delta \in (0, 1/e)$, with probability at least $1 - \delta$, 
\begin{align}
\sum_{t=1}^T [(s_t^n)^{\tp} D^n s_t^n - \tr(D^n \Sigma^n)] \leq \log(\frac{1}{\delta}) \tilde K \sqrt{T}. 
\end{align}
\end{lemma}

We now proceed in two steps:
\begin{itemize}
\item Step 1: Showing the connection between the auxiliary \texttt{SARL2} problem and the \texttt{MARL3} problem
\item Step 2: Using the \texttt{SARL2} problem to bound the regret of the \texttt{MARL3} problem
\end{itemize}

\subsubsection*{Step 1: Showing the connection between the auxiliary \texttt{SAR2} problem and the \texttt{MARL3} problem}
Similar to Lemma \ref{lem:connection_optimal_cost_appendix1}, we first state the following result.
\begin{lemma}
\label{lem:connection_optimal_cost_ext}
Let $J^{\diamond}(\theta_*^{1:4})$ be the optimal infinite horizon cost of the auxiliary \texttt{SARL2} problem, $J(\theta_*^{1:4})$ be the optimal infinite horizon cost of the \texttt{MARL3} problem, and $\Sigma^3$ and $\Sigma^4$ be as defined in Lemma \ref{lem:Sigma_convergence_appendix}. Then,
\begin{align}
J(\theta_*^{1:4}) = J^{\diamond}(\theta_*^{1:4}) + \tr(D^3 \Sigma^3) + \tr(D^4 \Sigma^4),
\label{eq:optimal_costs_relation_ext}
\end{align}
where we have defined $D^n:= Q^{nn} + (\tilde K^n(\theta_*^n))^{\tp} R^{nn} \tilde K^n(\theta_*^n)$, $n \in \{3,4\}$.
\end{lemma}

Next, similar to Lemma \ref{lem:expectation_equalities_appendix1}, we provide the following result.
\begin{lemma}
\label{lem:expectation_equalities_ext}
At each time $t$, the following equality holds between the cost under the policies of the \texttt{AL-SARL2} and the \texttt{AL-MARL2} algorithms,
\begin{align}
c(x_t,u_t) \vert_{\texttt{AL-MARL2}} &= c(x_t^{\diamond}, u_t^{\diamond})\vert_{\texttt{AL-SARL2}} + (e_t^3)^{\tp} D^3 e_t^3 +
(e_t^4)^{\tp} D^4 e_t^4,
\label{eq:expected_cost_cent_dec_ext}
\end{align}
where $e_t^n = x_t^n - \check x_t^n$ and $D^n= Q^{nn} + (\tilde K^n(\theta_*^n))^{\tp} R^{nn} \tilde K^n(\theta_*^n)$, $n \in \{3,4\}$.
\end{lemma}

\subsubsection*{Step 2: Using the \texttt{SARL2} problem to bound the regret of the \texttt{MARL3} problem}
In this step, we use the connection between the auxiliary \texttt{SARL2} problem  and our \texttt{MARL3} problem, which was established in Step 1, to prove Theorem \ref{lem:regret_bound_extension}. Similar to \eqref{eq:regret_inequality_case2},
\begin{align}
&R(T, \texttt{AL-MARL2}) = \sum_{t=1}^T \left[ c(x_t,u_t) \vert_{\texttt{AL-MARL2}} - J(\theta_*^{1:4})   \right] \notag \\
& = \sum_{t=1}^T  \left[c(x_t^{\diamond}, u_t^{\diamond}) \vert_{\texttt{AL-SARL2}} + (e_t^3)^{\tp} D^3 e_t^3 +
(e_t^4)^{\tp} D^4 e_t^4 \right]  -   \sum_{t=1}^T \left[ J^{\diamond}(\theta_*^{1:4}) + \tr(D^3 \Sigma^3) + \tr(D^4 \Sigma^4) \right] \notag \\
& = \sum_{t=1}^T  \left[c(x_t^{\diamond}, u_t^{\diamond}) \vert_{\texttt{AL-SARL2}} - J^{\diamond}(\theta_*^{1:4}) \right]  + \sum_{n=3}^4  \sum_{t=1}^T \left[ (e_t^n)^{\tp} D^n e_t^n - \tr(D^n \Sigma^n) \right]  \notag \\
& \leq R^{\diamond}(T, \texttt{AL-SARL2}) +2\log(\frac{1}{\delta}) \tilde K \sqrt{T}
 \label{eq:regret_inequality_case2_ext}
\end{align} 
where the second equality is correct because of  Lemma \ref{lem:connection_optimal_cost_ext} and Lemma \ref{lem:expectation_equalities_ext}. Further, if we define $v_t^n:= w_t^n$, $n \in \{3,4\}$, then $e_t^n$ has the same dynamics as $s_t^n$ in Lemma \ref{thm:prob_regret_appendix}. Then, the last inequality is correct because of Lemma \ref{thm:prob_regret_appendix}. %This proves the statement of Lemma \ref{lem:regret_bound_extension}.

{Now similar to Corollary \ref{cor:regret_bound_case}, by letting the \texttt{AL-SARL2} algorithm be the OFU-based algorithm of \cite{abbasi2011regret,ibrahimi2012efficient,faradonbeh2017regret, faradonbeh2019applications} or the TS-based algorithm of \cite{faradonbeh2017regret}, \texttt{AL-MARL2} algorithm achieves a $\tilde O(\sqrt{T})$ regret for the \texttt{MARL3} problem. This completes the proof.}

\section{Analysis and the results for unknown $\theta_*^1$ and $\theta_*^2$, two-way information sharing ($\gamma^1=\gamma^2=1$)}
\label{sec:two_way}

For the \texttt{MARL} of this section (it is called \texttt{MARL4} for future reference), we propose the \texttt{AL-MARL3} algorithm based on the \texttt{AL-SARL} algorithm.  \texttt{AL-MARL3} algorithm is a multi-agent algorithm which is performed independently by the agents. In the \texttt{AL-MARL3} algorithm, each agent has its own learner $\mathcal{L}$ and uses it to learn the unknown parameters $\theta_*^{1,2}$ of system 1. 

\begin{algorithm}[t]
   \caption{\texttt{AL-MARL3}} %\citep{ouyang2017learning}}
   \label{alg:TS_decentralized_2}
\begin{algorithmic}
\State \textbf{Input:} \texttt{agent_ID}, $x_0^1$, and $x_0^2$
  \State Initialize $\mathcal{L}$
   \For{$t=0,1,\ldots$}
\State Feed time $t$ and state $\vecc(x_t^1, x_t^2)$ to $\mathcal{L}$ and get $\theta_t^1 = [A_t^1, B_t^1]$ and $\theta_t^2 = [A_t^2, B_t^2]$
 \State Compute $K(\theta_t^{1,2})$ %from \eqref{eq:K_infinite}
   \If{$\texttt{agent_ID}=1$}
   \State Execute $u_t^1= K^1(\theta_t^{1,2}) \vecc(x_t^1, x_t^2)$
   \Else
  \State Execute $u_t^2= K^2(\theta_t^{1,2}) \vecc(x_t^1, x_t^2)$ 
   \EndIf
   \State  Observe new states $x_{t+1}^1$ and $x_{t+1}^2$
   \EndFor
\end{algorithmic}
\end{algorithm}

In this algorithm, at time $t$, agent $n$ feeds $\vecc(x_t^1, x_t^2)$ to its own \texttt{SARL} learner $\mathcal{L}$ and gets $\theta_t^1$ and $\theta_t^2$. Then, each agent $n$ uses $\theta_t^{1,2}$ to compute the gain matrix $K(\theta_t^{1,2})$ from \eqref{eq:K_infinite_appendix} and use this gain matrix to compute their actions $u_t^1$ and $u_t^2$ according to  the \texttt{AL-MARL3} algorithm. After the execution of the actions $u_t^{1}$ and $u_t^{2}$ by the agents, both agents observe the new state $x_{t+1}^1$ and agent 2 further observes the new states $x_{t+1}^2$. 

\begin{theorem}
\label{thm:regret_bound_case3}
Under Assumption \ref{assum:seed}, let $R(T, \texttt{AL-MARL3})$ be the regret for the \texttt{MARL4} problem under the policy of the \texttt{AL-MARL3} algorithm and $R^{\diamond}(T, \texttt{AL-SARL})$ be the regret for the auxiliary \texttt{SARL} problem under the policy of the \texttt{AL-SARL} algorithm. Then,
\begin{align}
R(T, \texttt{AL-MARL3})  = R^{\diamond}(T, \texttt{AL-SARL}).
\end{align} 
\end{theorem}

\begin{proof}
The proof simply results from the fact that under Assumption \ref{assum:seed}, the information that both agents have is the same, which reduces this problem to a \texttt{SARL} problem where an auxiliary agent plays the role of both agents.
\end{proof} 

\end{document}